\documentclass[a4paper,11pt]{article}

\input{macros}

\title{
\toptitlebar
{{\center\baselineskip 18pt
                      {\Large\bf On the Role of Batch Size \\ in Stochastic Conditional Gradient Methods}}
} 
\bottomtitlebar}
\date{}
\author{
Rustem Islamov\textsuperscript{1,4,$\dagger$}, 
Roman Machacek\textsuperscript{2},
Aurelien Lucchi\textsuperscript{1},
Antonio Silveti-Falls\textsuperscript{3} 
\newline
Eduard Gorbunov\textsuperscript{4,$\star$}, 
Volkan Cevher\textsuperscript{5,$\star$}
}
\affil{
\textsuperscript{1}University of Basel, Switzerland, \quad
\textsuperscript{2}University of Bern, Switzerland
}
\affil{
\textsuperscript{3}CentraleSupélec, France,
\quad 
\textsuperscript{4}MBZUAI, UAE,
\quad 
\textsuperscript{5}EPFL, Switzerland
}

\begin{document}

\maketitle

\def\thefootnote{$\dagger$}\footnotetext{Most of this work was done when Rustem Islamov was a visiting student in the group of Prof. Eduard Gorbunov at MBZUAI, UAE.}
\def\thefootnote{$\star$}
\footnotetext{The last two authors share senior authorship.}
\def\thefootnote{\arabic{footnote}}

\begin{abstract}
We study the role of batch size in stochastic conditional gradient methods under a $\mu$-Kurdyka–Łojasiewicz ($\mu$-KL) condition. Focusing on momentum-based stochastic conditional gradient algorithms (e.g., Scion), we derive a new analysis that explicitly captures the interaction between stepsize, batch size, and stochastic noise. Our study reveals a regime-dependent behavior: increasing the batch size initially improves optimization accuracy but, beyond a critical threshold, the benefits saturate and can eventually degrade performance under a fixed token budget. Notably, the theory predicts the magnitude of the optimal stepsize and aligns well with empirical practices observed in large-scale training. Leveraging these insights, we derive principled guidelines for selecting the batch size and stepsize, and propose an adaptive strategy that increases batch size and sequence length during training while preserving convergence guarantees. Experiments on NanoGPT are consistent with the theoretical predictions and illustrate the emergence of the predicted scaling regimes. Overall, our results provide a theoretical framework for understanding batch size scaling in stochastic conditional gradient methods and offer guidance for designing efficient training schedules in large-scale optimization. 
\end{abstract}

\section{Introduction}
Large-scale language model training is constrained by a {token budget} $T$ rather than by a fixed number of optimization steps.
In this regime, we face a familiar {batch size tradeoff}: increasing the batch size $B$ improves hardware utilization, yet beyond a certain scale it can degrade optimization efficiency and hurt generalization~\citep{goyal2017accurate,keskar2017largebatch,smith2018dontdecay,shallue2019measuring}.

A token budget-aware viewpoint makes this tradeoff explicit.
With batch size $B$ and sequence length $S$, the number of parameter updates is
$
K \eqdef \frac{T}{BS},
$
and hence $(B,S)$ and the stepsize jointly determine how effectively the token budget is converted into optimization progress.
This coupling raises a central question in model training:
\emph{how should $(B,S)$ and the stepsize be chosen, and adapted, to optimize performance under a fixed token budget $T$?}

Recent empirical studies have further refined this picture. In particular, critical batch sizes -- the point at which scaling $B$ stops being beneficial -- appear to scale primarily with the effective data size and only weakly with model size under a fixed token budget~\citep{zhang2024cbs,bergsma2025powerlines}. Additionally, the critical batch threshold is often stage-dependent, motivating warmup and stage-wise training schedules~\citep{merrill2025critical}. Taken together, these findings suggest that the batch size should be treated as a dynamic optimization variable rather than a fixed hyperparameter. However, these insights remain largely empirical: they do not provide explicit optimization error laws as functions of $(B,S,T)$, nor do they characterize when increasing batch size becomes provably detrimental under a fixed token budget.

In parallel, hyperparameter transfer frameworks such as $\mu$P have shown that, with appropriate parameterization and initialization, gradient magnitudes can be kept $\Theta(1)$ across model scales, enabling stable training without retuning learning rates~\citep{yang2020feature,yang2021tuninglarge,yang2022tensor}.
However, these results are inherently local: they ensure that individual updates neither explode nor vanish, but do not address how batch size, sequence length, and stepsize should scale \emph{globally} with the token budget.

Our work bridges this gap by showing that hyperparameters that are locally optimal for a given $(B,S,T)$ can become provably suboptimal as the token budget increases, even under $\mu$P-style initialization. To obtain such global scaling laws, we derive an analysis for stochastic conditional gradient (SCG) methods~\citep{pethick2025scion}, a projection-free framework that underlies several modern norm-constrained training algorithms. This class of algorithms is closely aligned with modern optimizers such as Muon \citep{jordan2024muon}.

Our analysis is carried out for stochastic optimization \eqref{eq:problem1} under smoothness~\eqref{eq:smoothness} in a general norm, norm equivalence~\eqref{eq:norm_equiv}, and a $\mu$-Kurdyka--\L{}ojasiewicz ($\mu$-KL) error bound~\eqref{eq:mu_kl}~\citep{karimi2016pl,bolte2007lojasiewicz}.
The $\mu$-KL condition is particularly well matched to SCG geometry, as it relates first-order stationarity to suboptimality measured in the dual norm induced by the linear minimization oracle (LMO).

Specializing our convergence bounds to the fixed-token setting $T=KBS$ yields an explicit, non-monotone dependence of the achievable optimization error on the effective batch–sequence scale $BS$.
Three regimes emerge:
\emph{(i)} a noise-dominated regime where increasing $BS$ improves performance,
\emph{(ii)} an intermediate regime where the best achievable error is essentially independent of $BS$, and
\emph{(iii)} a large-batch regime where performance deteriorates as $BS$ grows under a fixed token budget.

Balancing the dominant terms yields a \emph{critical} effective batch--sequence--token (BST) scale rule
$
BS \asymp T^{2/3}
$
up to problem-dependent factors that we derive in this work,
revealing how curvature, noise, geometry, and error-bound strength shift the optimal operating point.
Importantly, our analysis shows that {large batch sizes do not inherently degrade performance}:
when batch size, sequence length, and learning rate are chosen according to our BST scaling rule, large-batch training remains effective and token-efficient. In contrast to $\mu$P, our perspective disentangles {local stability}, as controlled by parameterization and initialization, from {global efficiency}, as governed by token-budget–aware optimization.

Our contributions are as follows: \\[-7mm]
\begin{itemize}
    \item \textbf{Convergence guarantees for momentum SCG under $\mu$-KL.}
    We establish convergence guarantees for \Cref{alg:spectral_gd_decay_fw} under the $\mu$-KL condition~\eqref{eq:mu_kl} in a general normed geometry, explicitly tracking the effects of momentum, smoothness, and stochastic gradient noise.
    Our bounds hold \emph{in expectation} under bounded-variance and $L$-smoothness assumptions.

    \item \textbf{A token-budget view of batch, sequence length, and stepsize scaling.}
    By translating iteration complexity into token complexity via $T=KBS$, we obtain explicit $(B,S,T)$-dependent error laws and identify the \emph{critical} effective batch size $BS$ that separates beneficial from harmful scaling. \\[-6mm]

    \item \textbf{Actionable adaptive scheduling rules.}
    We turn the theory into concrete recipes for choosing and \emph{updating} $(\beta,B,S)$ during training under a fixed token budget, yielding the scaling relations \eqref{eq:how_to_scale_batch_size}--\eqref{eq:how_to_scale_stepsize} and a two-stage (and more generally multi-stage) protocol validated empirically on NanoGPT (\textit{cf.},\ \Cref{fig:variance}).
\end{itemize}

Our results complement classical large-batch heuristics such as linear learning rate scaling with warmup \citep{goyal2017accurate} and adaptive batch size schedules \citep{smith2018dontdecay}, while offering a \emph{projection-free} viewpoint rooted in conditional gradient geometry.
They are also consistent with empirical observations that there exists a largest useful batch size depending on training stage and problem statistics \citep{mccandlish2018empirical,shallue2019measuring}, and provide an explicit optimization-side mechanism for the ``too-large batch hurts'' regime under a fixed token budget.

\section{Related Works}
\paragraph{Assumptions in SCG methods: smoothness.} Convergence analyses for stochastic conditional gradient (SCG) (aka Frank--Wolfe) methods and, more broadly, \emph{LMO-based} methods, have been conducted under various assumptions. Most analyses, including our analysis, assume standard $L$-smoothness. However, recent works consider relaxed notions, such as $(L_0,L_1)$-smoothness \citep{zhang2019gradient} and other extensions beyond global smoothness \citep{pethick2025generalized,riabinin2025gluon}. 
Extending our analysis to these generalized smoothness settings is an interesting direction for future work, but it lies beyond the scope of the present paper.

\paragraph{Assumptions in SCG methods: structured nonconvexity.}
Most prior work considers either general nonconvex or (strongly) convex objectives, failing to capture practical learning rate and batch size scaling effects observed in large-scale training. This limitation motivates our study under structured nonconvexity.

Several recent works study structured nonconvexity for LMO-based or related methods. \citet{yang2024adaptive} derives an analysis under a generalized Polyak–Łojasiewicz condition, which recovers our \Cref{asmp:mu_kl} as a special case. Their method, however, does not use momentum and assumes almost surely affine bounded noise, in contrast to the bounded variance setting considered here.

\citet{kovalev2025orthogonalization} studies stochastic conditional gradient methods under star-convexity, a condition closely related to the $\mu$-KL condition. However, our work empirically validates the $\mu$-KL condition in large-scale language models training and uses it to derive a principled BST scaling rule under a fixed token budget. Finally, \citet{riabinin2025gluon} study an LMO-based method with adaptive layer-wise learning rates under the classical Polyak-{\L}ojasiewicz (PL) condition \citep{polyak1963gradient, lojasiewicz1963topological}, restricted to the deterministic setting without momentum, limiting its applicability to the large-scale stochastic settings.

\paragraph{Works on Hyperparameter Transfer.} 
Transferring hyperparameters (HPs) tuned on small proxy models to large-scale training has become increasingly important as model sizes grow.
This line of work was initiated by the $\mu$P framework~\citep{yang2020feature,yang2021tuninglarge,yang2022tensor}, which enables zero-shot transfer of learning rates across model \emph{width}, and was later extended to other aspects of the model architecture, such as depth~\citep{yang2023feature,dey2025don}. 

Technically, $\mu$P-style analyses focus on parameterizations that ensure gradient magnitudes and parameter updates remain $\Theta(1)$ around initialization. These analyses assume a fixed number of tokens processed per step and do not characterize optimization behavior when the number of optimization steps is significantly larger than the model width.

To reason about the latter regime, we analyze SCG methods under a $\mu$-KL condition and derive convergence guarantees that explicitly depend on the batch size $B$, sequence length $S$, and total token budget $T$. This trajectory-level analysis allows us to characterize how optimization error accumulates as a function of $(B,S,T)$ and to derive principled scaling rules for jointly adapting batch size, sequence length, and stepsize, in contrast to prior hyperparameter transfer works that focus on local, per-step stability governing early training behavior.

Closely related in spirit, \citet{wang2024set} study hyperparameter transfer specifically for AdamW by identifying an EMA timescale induced by decoupled weight decay and deriving scaling rules for the weight-decay coefficient across token budget. Our work addresses a complementary question for SCG methods: how the effective batch–sequence scale and Frank--Wolfe stepsize should co-vary with the token budget.

\paragraph{Batch Size Scheduling.} 

Adapting the batch size during training is a well-established and practical strategy, motivated by both computational efficiency and optimization dynamics. Increasing the batch size can serve as an alternative to learning rate decay, reduce the number of parameter updates, and improve parallel utilization~\citep{smith2018dontdecay}. However, compared to small-batch training, large batches often lead to worse generalization performance and tend to converge to sharper minima~\citep{keskar2017largebatch}.

A complementary empirical view suggests a \emph{critical batch size} (CBS), beyond which increasing $B$ yields diminishing token efficiency; \citet{mccandlish2018empirical} relate CBS to the gradient noise scale and argue that it evolves during training. In the LLM setting, scaling-law work \citep{kaplan2020scaling,hoffmann2022training} primarily addresses how to allocate a fixed compute budget across model size and training tokens, rather than prescribing within-run batch size schedules. More recently, \citet{bi2024deepseek} report empirical power-law relations between compute budget, batch size, and learning rate that perform well at scale.

Taken together, these works reinforce a central practical message: the best batch size is typically not a fixed constant, but depends on the training stage, optimization hyperparameters, and budget. Motivated by this, we seek \emph{principled, token-budget--aware} rules that characterize how the {optimal} effective batch--sequence scale and stepsize should co-vary with $T$, and how $(B,S,\beta)$ should be adapted.

\begin{algorithm}[tb]
  \caption{Stochastic Conditional Gradient (SCG)}
  \label{alg:spectral_gd_decay_fw}
  \begin{algorithmic}
    \STATE {\bfseries Input:} $x_0,m_0 \in \cX$, parameters $\alpha, \beta \in (0,1), \eta > 0$ 
    \FOR{$k=0, \ldots, K-1$}
    \STATE sample $\xi_k\sim\cD$
    \STATE compute $m_{k+1} = (1-\alpha)m_k + \alpha g(x_k;\xi_k)$
    \STATE compute $d_{k+1} = {\rm arg}\min_{d\in\cX}\<m_{k+1},d>$ s.t. $\|d\|\le 1$ 
    \STATE compute $x_{k+1} = (1-\beta)x_k + \beta\eta d_{k+1}$
    \ENDFOR
  \end{algorithmic}
\end{algorithm}

\section{Problem Formulation and Assumptions}\label{sec:problem}

We consider the following problem template:

\begin{equation}\label{eq:problem1}
    \min_{x\in\cX} f(x),
\end{equation}
where the space $\cX$ is equipped with a standard Euclidean norm $\|\cdot\|_2$ induced by the inner product $\<\cdot,\cdot>$, i.e., $\|x\|_2=\sqrt{\<x,x>}$, and another norm $\|\cdot\|$, which possibly does not coincide with the Euclidean one. For the norm $\|\cdot\|$, we define the associated dual norm $\|x\|_* \eqdef \sup_{\|x'\|\le 1}\<x,x'>$ for all $x\in\cX$. We seek to solve \eqref{eq:problem1} using \Cref{alg:spectral_gd_decay_fw}. 

\begin{assumption}\label{asmp:smoothness}
Let the gradient $\nabla f(\cdot)$ be Lipschitz continuous with respect to the norm $\|\cdot\|$:
\begin{equation}\tag{A1}\label{eq:smoothness}
    \|\nabla f(x)-\nabla f(x')\|_* \le L\|x-x'\| \quad \text{for all } x,x'\in\cX,
\end{equation}
where $L>0$ is the gradient Lipschitz constant.

\end{assumption}

\begin{assumption}\label{asmp:norm_equiv} There exist a constant $\rho > 0$ such that 
\begin{equation}\tag{A2}\label{eq:norm_equiv}
    \|x\|_* \le \rho\|x\|_2 \quad \text{for all } x\in\cX.
\end{equation}
Note that such a constant always exists by norm equivalence, which always holds in finite-dimensional spaces $\cX$.

\end{assumption}

\begin{assumption}\label{asmp:mu_kl} The objective function $f(x)$ is $\mu$-KL for some $\mu > 0$:
\begin{align}\tag{A3}\label{eq:mu_kl}
    \|\nabla f(x)\|_* \ge \mu(f(x)-f^{\star}) \quad \text{for all } x \in \cX
\end{align} 
where $f^{\star} = \min_{x\in\cX} f(x)$.
\end{assumption}

Note that condition \eqref{eq:mu_kl} is closely related to the Polyak-{\L}ojasiewicz (PL) condition $\|\nabla f(x)\|_2^2 \ge \mu(f(x)-f^{\star})$ originally studied in \citet{polyak1963gradient, lojasiewicz1963topological}. Variants of the PL condition have been investigated for over-parameterized models \citep{liu2022loss}. A key distinction between the $\mu$-PL and $\mu$-KL conditions lies in the exponent of the gradient norm, making the difference between them significant when the norm is small.

\begin{figure}
    \centering
    \begin{tabular}{c}
    \includegraphics[width=0.4\linewidth]{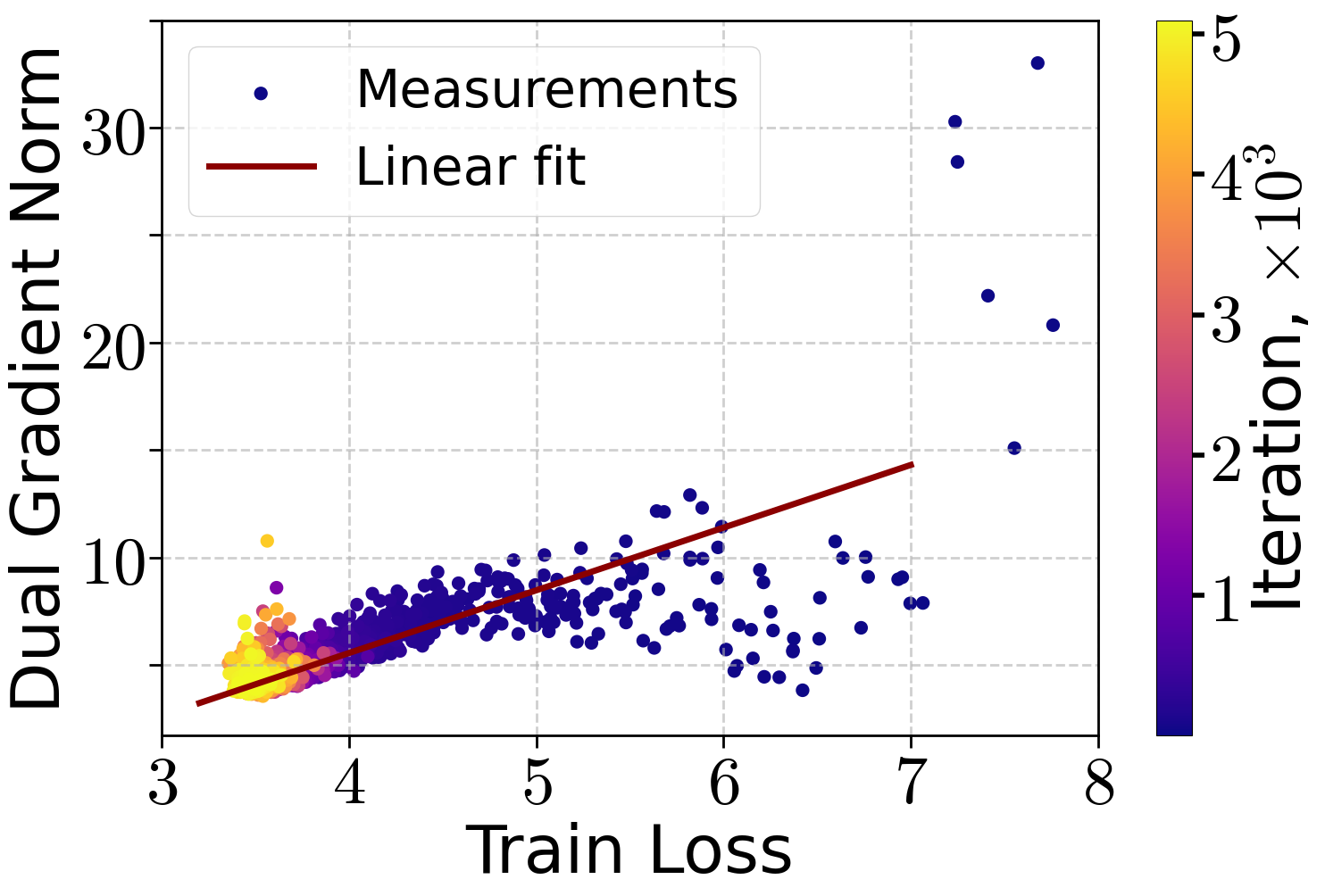} 
    
    \end{tabular}
    \caption{Empirical verification of the validity of \Cref{asmp:mu_kl} during the training of a 124M NanoGPT model. The points with a loss below 5 fit a linear function well, with a slope equal to $\mu$.}
    \label{fig:mu_kl_verification}
\end{figure}

Nevertheless, condition~\eqref{eq:mu_kl} has been extensively used in the optimization literature to analyze gradient descent under the Euclidean norm \citep{bolte2014proximal, fatkhullin2022sharp}. For problems with a bounded domain,
the $\mu$-KL condition is closely related to $\zeta$-quasar convexity ($\zeta$-QC) \citep{hardt2018gradient, 
guminov2017accelerated}, which requires $\langle \nabla f(x), x - x^{\star} \rangle \geq \zeta(f(x) - f^{\star})$  for some $x^{\star}\in\cX$ and all $x \in \cX$. $\zeta$-QC  naturally arises in the training of neural networks \citep{zhou2019sgd, kleinberg2018alternative}. When $\cX$ is bounded with diameter $R$ with respect to the norm $\|\cdot\|$, $\zeta$-QC  implies the $\mu$-KL condition with $\mu = \zeta / R$.

In this work, we extend the applicability of the standard $\mu$-KL assumption beyond the Euclidean norm. To demonstrate its validity in practice, we track the train loss and dual gradient norm during the training of a 124M NanoGPT model. In \Cref{fig:mu_kl_verification}, we observe that the measurements fit a linear function well, especially when the loss is below 5 
(\textit{cf.}, the description of the full setting in \Cref{sec:experiments_mu_kl})

We make the assumption below for the gradient noise.

\begin{assumption}\label{asmp:bounded_variance} We have access to the unbiased estimator $g(\cdot;\xi)\colon \cX \to \cX$ of the gradient $\nabla f(\cdot)$, where $\xi\sim\cD$ is a random variable sampled from a probability distribution $\cD$. We assume that the stochastic gradient estimator $g(\cdot;\xi)$ is unbiased and has $\sigma$-bounded variance for some $\sigma \ge 0$:
\begin{align}\tag{A4}\label{eq:bounded_variance}
\EE_{\xi\sim\cD}[g(x;\xi)] = \nabla f(x) \quad \text{and} \quad \EE_{\xi\sim\cD}[\|g(x;\xi)-\nabla f(x)\|_2^2] \le \sigma^2 \quad \text{for all } x\in\cX.\notag
\end{align}
Additionally, let $\sigma^2 = \frac{\sigma_\star^2}{BS}$, where $B$ and $S$ are batch size and sequence length respectively.
\end{assumption}

\Cref{asmp:bounded_variance} is a classical assumption for the in-expectation convergence analysis of stochastic methods \citep{ghadimi2012optimal, ghadimi2013stochastic}.
We verify the validity of \Cref{asmp:bounded_variance} during the training in \Cref{fig:variance} (\textit{cf.}, the description of the full setting in \Cref{sec:experiments_variance}).

\section{Theoretical Analysis}\label{sec:theorical_analysis}

This section establishes convergence guarantees for \Cref{alg:spectral_gd_decay_fw}, guiding how to choose the batch size $B$, sequence length $S$, and stepsize $\beta$ under a fixed token budget $T$. The proof and the full statement of the following theorem are deferred to \Cref{apx:inexp_convergence_proofs_no_restarts}.

\begin{theorem}
\label{thm:str_decay_mu_kl_expectation_no_restarts} Let Assumptions \eqref{eq:smoothness}, \eqref{eq:norm_equiv}, \eqref{eq:mu_kl}, and \eqref{eq:bounded_variance} hold. Let $m_{0} = g(x_{0};\xi_{0})$. 
Let the parameters of 
\Cref{alg:spectral_gd_decay_fw} and initialization $x_0$ be chosen as follows
\begin{gather}
    \beta = \cO\left(\frac{1}{K}\right), 
    \quad 
    \eta = \wtilde{\cO}\left(\frac{1}{\mu}\right), \quad \alpha = \min\left\{1, {\cO}\left(\frac{(\varepsilon\mu)^2}{(\rho\sigma)^2}\right) \right\}, 
    \quad
    2\|x_0\| \le \eta, \quad \text{and}\notag\\ 
    K = \max\left[
    \wtilde{\cO}(1),
   \wtilde{\cO}\left( \max\left\{\frac{L}{\varepsilon\mu^2},
    \frac{\rho\sigma}{\varepsilon\mu},
    \frac{L(\rho\sigma)^2}{\mu(\varepsilon\mu)^3},
    \frac{(\rho\sigma)^3}{(\varepsilon\mu)^3}\right\}\right)
    \right],
\end{gather}
where $\cO$ hides all numerical constants and $\tilde{\cO}$ hides all numerical and logarithmic factors. Then, the output of \Cref{alg:spectral_gd_decay_fw} after $K$ iterations satisfies $\EE[f(x_K) - f^{\star}] \le \varepsilon$.
\end{theorem}

\begin{remark}\label{remark:on_star_convexity_res}
    Convergence bounds for SCG were derived in \citet{pethick2025scion} for the Frank--Wolfe gap, then similar results to \Cref{thm:str_decay_mu_kl_expectation_no_restarts} were given by \citet{kovalev2025orthogonalization}\footnote{\citet{kovalev2025orthogonalization} studies a stochastic first-order non-Euclidean trust-region method with momentum and weight decay, which is equivalent to \Cref{alg:spectral_gd_decay_fw}.} under star-convexity, a special case of $\zeta$-quasar convexity with $\zeta = 1$. In light of the relationship between the $\mu$-KL condition and $\zeta$-QC in Section~\ref{sec:problem}, this similarity is expected. 
    
    Our work goes beyond this connection in two important ways. First, we provide empirical justification for the use of the $\mu$-KL condition in the analysis. Second, building on this framework, we derive new theory-guided scaling rules for both the learning rate and the batch size.
\end{remark}

 In practice, the number of iterations $K$ cannot be arbitrarily large. In fact, $K$ is trivially constrained by the available token budget $T$, the two being related by the simple identity $T=K\cdot B\cdot S$. Consequently, the requirement on $K$ in \Cref{thm:str_decay_mu_kl_expectation_no_restarts}\footnote{We ignore the requirement $K = \widetilde{\cO}(1)$, as it is always satisfied in practice; see also \Cref{cor:best_eps_appendix} for the details. In the sequel, we omit numerical constants for clarity.} can be equivalently expressed as a condition on $T$ by multiplying both sides by $BS$:
\begin{equation*}
    T = \tilde{\cO}\left(\max\left\{
    \frac{LBS}{\varepsilon\mu^2}, \frac{\rho\sigma BS}{\varepsilon\mu}, 
    \frac{L(\rho\sigma)^2BS}{\mu(\varepsilon\mu)^3}, 
    \frac{(\rho\sigma)^3BS}{(\varepsilon\mu)^3}\right)\right)
\end{equation*}
Under a fixed token budget, the expression above indicates that we cannot achieve an arbitrary optimization error $\varepsilon$. Instead, \Cref{cor:best_eps} lower bounds the achievable error.

\begin{corollary}[BST Scaling Rule]\label{cor:best_eps}
Under the setup of \Cref{thm:str_decay_mu_kl_expectation_no_restarts}, running the algorithm with parameters from \Cref{thm:str_decay_mu_kl_expectation_no_restarts} for $\frac{T}{BS}$ iterations, we achieve the optimization error 
\begin{equation}\label{eq:best_eps}
    \varepsilon = \tilde{\cO}\left(\max\left\{
    \frac{LBS}{\mu^2T},  
    \left(\frac{L\rho^2\sigma_\star^2}{\mu^4T}\right)^{ 1/3}, 
    \frac{\rho\sigma_\star}{\mu(T^2BS)^{1/6}}\right\}\right),
\end{equation}
where $\tilde{\cO}$ hides numerical and logarithmic factors.
\end{corollary}

\begin{figure}
    \centering
    \begin{tabular}{cc}
    \includegraphics[width=0.4\linewidth]{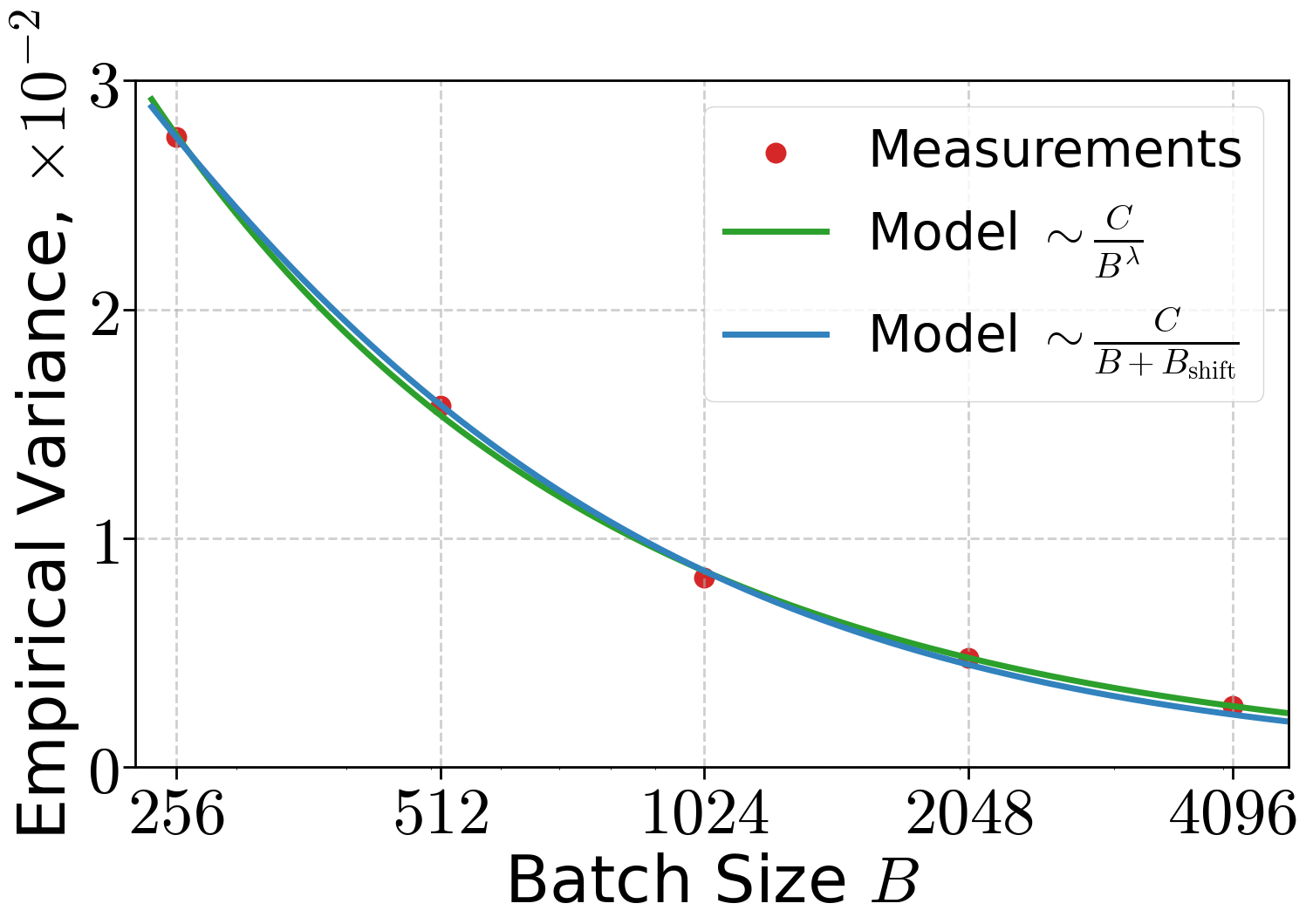} & 
    \includegraphics[width=0.4\linewidth]{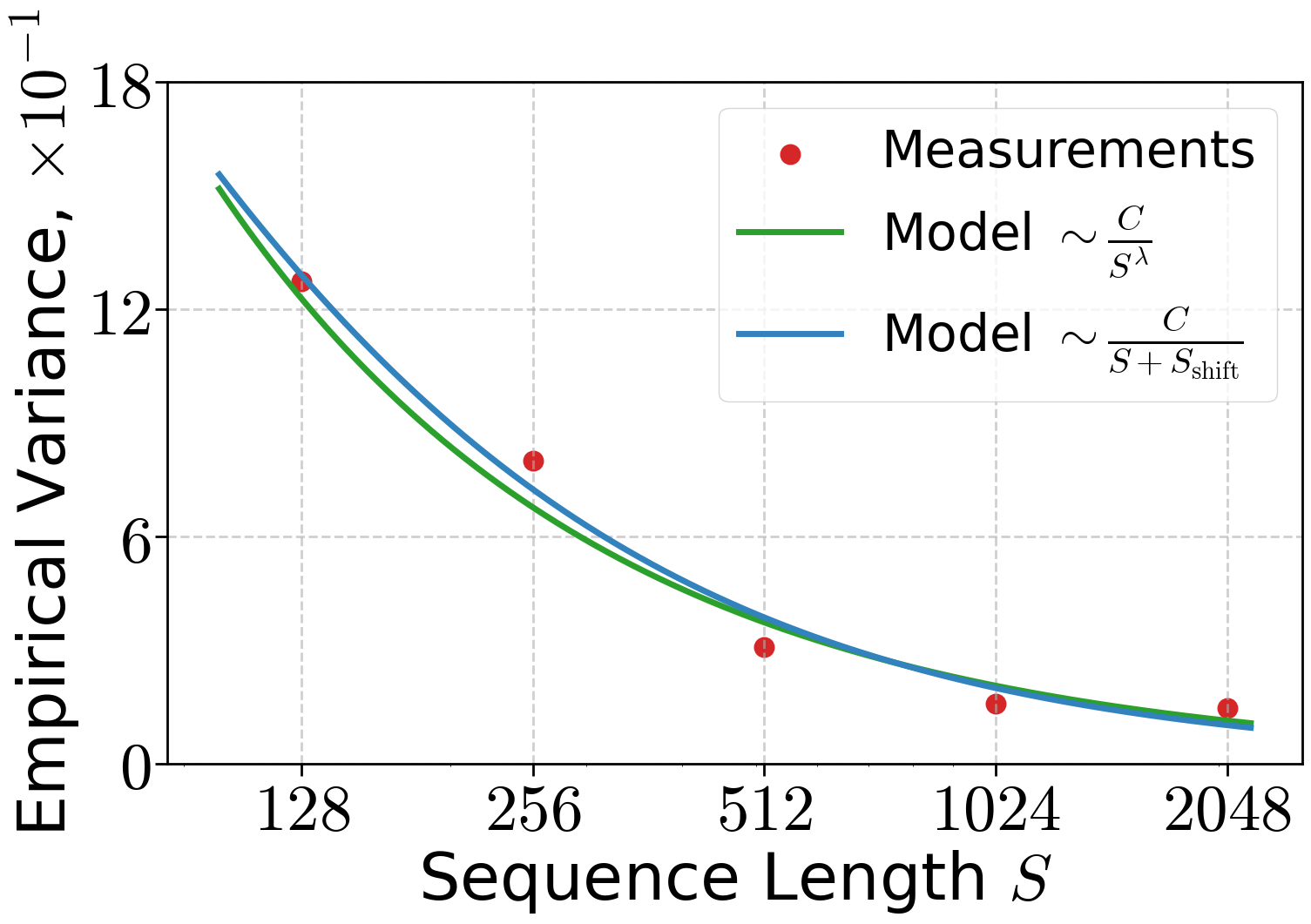}
    \end{tabular}
    \caption{Empirical gradient variance and fitted power-law models as functions of batch size $B$ with fixed sequence length $S=1024$ ({\bf left}) and sequence length $S$ with fixed batch size $B=512$ ({\bf right}) when training a 124M NanoGPT model on the FineWeb dataset under a fixed token budget $T=2.7$B. For the left plot, the estimated scaling exponent is $\lambda \approx 0.9$ and $B_{\rm shift}\approx90$, while for the right plot they are $\lambda \approx 1.1$ and $S_{\rm shift} \approx 35$. The fitted models support the validity of \Cref{asmp:bounded_variance}.}
    \label{fig:variance}
\end{figure}

\Cref{cor:best_eps} provides key insights into how the error $\varepsilon$ typically varies as the product $BS$ changes:
\begin{enumerate}
    \item For small batch sizes, the third term in \eqref{eq:best_eps} dominates, and $\varepsilon$ improves as $BS$ increases.\\[-6mm]

    \item When the batch size exceeds $\left(\frac{\mu \rho \sigma_\star}{L}\right)^2$, the second term in \eqref{eq:best_eps} becomes dominant. In this regime, the error is independent of the batch size and sequence length, and the error instead scales as $\sim T^{-1/3}$.\\[-6mm]
    
    \item Further increasing the batch size moves the system into an iteration-starved regime where the first term dominates, causing the error to deteriorate linearly. 
\end{enumerate}

\Cref{cor:best_eps} indicates that the optimal achievable performance lies in the second regime, where the optimization error $\varepsilon$ is independent of both the batch size and the sequence length. From a practical perspective, however, larger batch sizes are often preferred to improve GPU utilization \citep{narayanan2021efficient}. 

This motivates us to select the batch size and sequence length at the crossover between the second and third regimes. Following this intuition, we choose $B$ and $S$ as

\begin{equation}\label{eq:critical_batch_size}
\hspace{-3mm}\frac{L}{\mu^2}\frac{BS}{T}
=
\left(\frac{L\rho^2\sigma_\star^2}{\mu^4 T}\right)^{1/3}
\Leftrightarrow
BS
=
\left(\frac{T \mu \rho \sigma_\star}{L}\right)^{2/3},
\end{equation}
balancing final performance and hardware efficiency.
Next, the BST rule results in the Frank–Wolfe stepsize
\begin{equation}\label{eq:critical_stepsize}
\beta_\star \sim \frac{1}{K}.
\end{equation}
Notably, a Frank–Wolfe stepsize of this form is  used in practice when employing decoupled weight decay \citep{loshchilov2019decoupled} to train LLMs near the Chinchilla-optimal token-per-parameter (TPP) regime \citep{xiao2024rethinking,qiu2025hyperparameter}, where the model depth scales proportionally with the token budget.

Using $\varepsilon=\left(\frac{L\rho^2\sigma_\star^2}{\mu^4T}\right)^{1/3}$, $BS = \left(\frac{T\mu\rho\sigma_\star}{L}\right)^{2/3}$, and \Cref{asmp:bounded_variance} in \Cref{thm:str_decay_mu_kl_expectation_no_restarts}, we obtain that the momentum parameter $\alpha$ 
\begin{equation*}
    \alpha \sim \frac{\mu^2BS}{\rho^2\sigma_\star^2}\cdot \left(\frac{L\rho^2\sigma_\star^2}{\mu^4T}\right)^{2/3} = \left(\frac{L}{\mu\rho\sigma_\star T}\right)^{2/3} \left(\frac{T \mu \rho \sigma_\star}{L}\right)^{2/3}=\text{Const}.
\end{equation*} 
This suggests that if we find an optimal momentum parameter $\alpha$ for a small model, under the BST scaling rule, it transfers to the larger setting.

To summarize, the BST scaling rule suggests the following choice of parameters in \Cref{alg:spectral_gd_decay_fw}:

\begin{findingbox}
\begin{equation}\label{eq:BST}
    BS \sim T^{2/3}, \quad \beta \sim \frac{1}{K}, \quad \alpha = \text{Const}.
\end{equation}
\end{findingbox}
In \Cref{sec:strategies}, we provide a more detailed explanation of the BST scaling rule for the parameter choice when training under a fixed TPP or increasing the token budget for the same model.

    \section{Strategies for Hyperparameter Choice}\label{sec:strategies}

    \paragraph{Training Setup.}
    Training a model such that \eqref{eq:critical_batch_size} holds establishes working strategies on how to train a larger model of size $D_{\rm 1}$ efficiently, given that we have a tuned configuration (i.e., the tuned values of Frank--Wolfe stepsize $\beta_{\rm 0}$, momentum parameter $\alpha$, batch size $B_{\rm 0},$ and sequence length $S_{\rm 0}$) for a smaller model of size $D_{\rm 0}$. We consider the training under a fixed TPP, which implies that the available token budget increases proportionally to the model size, i.e., $\nicefrac{T_{\rm 1}}{T_{\rm 0}} = \nicefrac{D_{\rm 1}}{D_{\rm 0}}.$ Moreover, we assume that the problem constants $L=L(D), \mu=\mu(D),$ and $\rho=\rho(D)$ change with model size. We denote the constants with subscripts $1$ and $0$ for models of size $D_{\rm 1}$ and $D_{\rm 0}$, respectively.

    \begin{remark}
       In this work, we assume that the variance constant $\sigma_\star^2$ in \Cref{asmp:bounded_variance} does not depend on the model size, as estimating its scaling with model size is computationally infeasible. We acknowledge, however, that in practice $\sigma_\star^2$ may change as the model size grows. 
    \end{remark}

    \begin{remark}\label{rem:schedule}
        \Cref{thm:str_decay_mu_kl_expectation_no_restarts} and the BST Scaling Rule \eqref{eq:BST} are given for a constant Frank--Wolfe stepsize $\beta = c/K$. In our experiments we instead use a constant stepsize followed by a linear rampdown over the last $28\%$ of the token budget (\Cref{tab:training_details}), with the max of this schedule always fixed at $\beta$; specifically, $\beta_0$ above is the tuned max. A different decay shape or rampdown fraction could change the constants in \eqref{eq:how_to_scale_batch_size} and \eqref{eq:how_to_scale_stepsize} and therefore shift the empirical optima in \Cref{tab:batch_size_abl_124M} and \Cref{fig:lr_transfer_124M,fig:alpha_lr_transfer}. Thus, these values should be interpreted as tied to this schedule. The estimates of $L$ and $\rho$ in \Cref{sec:constant_estimation}, averaged over the last $100$ iterations, are likewise specific to this schedule.
    \end{remark}

\begin{table*}[t]
\centering
\caption{Summary of BST scaling rules for different optimizers when transferring hyperparameters to larger token budget/model size.}
\label{tab:practicioners_guidance}
\vspace{-2mm}
\resizebox{\linewidth}{!}{
    \begin{tabular}{c|ccc|ccc|c}
        \toprule
        
        \multirow{2}{*}{{\bf Regime}} & \multicolumn{3}{c|}{\bf Tuned at $\mathbf{T_0}$} & 
        \multicolumn{3}{c|}{\bf Transfer Rule to $\mathbf{T_1}$} &
        \multirow{2}{*}{$\boldsymbol{\chi_{0\to 1}}$}
        \\ \cline{2-7}
        & {\bf FW Stepsize / LR} 
        & {\bf Batch Size} 
        & {\bf Momentum}
        & {\bf FW Stepsize / LR} 
        & {\bf Batch Size} 
        & {\bf Momentum}
        & \\
        \midrule
        {\bf Increasing} &
        \multirow{5}{*}{$\beta_0^\star$} & 
        \multirow{5}{*}{$B_0^\star S_0^\star$} &
        \multirow{5}{*}{$\alpha$} &
        \multirow{5}{*}{$\beta_1=\beta_0^\star {\color{darkred}\left(\frac{T_1}{T_0}\right)^{-1/3}}\chi_{0\to 1}^{2/3}$} &
        \multirow{5}{*}{$B_1S_1 = B_0^{\star}S_0^{\star}{\color{darkred}\left(\frac{T_1}{T_0}\right)^{2/3}}\chi_{0\to 1}^{2/3}$} &
        \multirow{5}{*}{$\alpha$} &
        \multirow{2}{*}{$\frac{\rho_1}{\rho_0}$}
        \\
        {\bf Token Budget} & & & & & & & \\
        \cline{1-1}\cline{8-8}
        {\bf Increasing} & & & & & & &
        \multirow{3}{*}{$\frac{\mu_1\rho_1 L_0}{\mu_0\rho_0 L_1}$}
        \\
        {\bf Token Budget and} & & & & & & & \\
        {\bf Model Size} & & & & & & & \\
        \bottomrule
    
    \end{tabular}
    } 
    \vspace{-3mm}
\end{table*}

    \subsection{Increasing Batch Size}

    We assume that the optimal batch size $B_0^\star$, sequence length $S_0^\star$, and $\beta_0^\star$ are tuned for a small model\footnote{Ideally, we want all hyperparameters of the optimizer and model to be tuned for a small model, including radii $\eta$ or the initialization. However, such a task is infeasible even for a small model. Therefore, we focus on the main hyperparameters that affect the final performance the most: batch size, sequence length, and Frank--Wolfe stepsize, while we set the rest according to default values obtained from prior work.} of size $D_{\rm 0}$ and satisfy \eqref{eq:critical_batch_size}, namely 
    $$B_{\rm 0}^\star S_{\rm 0}^\star \sim \left(\frac{T_{\rm 0}\mu_{\rm 0}\rho_{\rm 0}\sigma_\star}{L_{\rm 0}}\right)^{2/3}.$$
    We now determine $B_1$ and $S_1$ such that \eqref{eq:critical_batch_size} remains satisfied for a larger model. By simple manipulation, we have
    \begin{gather}\label{eq:how_to_scale_batch_size}
        B_{\rm 1}S_{\rm 1} 
        = B_{\rm 0}^{\star}S_{\rm 0}^{\star} \left(\frac{\frac{T_{\rm 1}}{T_{\rm 0}} \frac{\mu_{\rm 1}}{\mu_{\rm 0}}   \frac{\rho_{\rm 1}}{\rho_{\rm 0}}}{\frac{L_{\rm 1}}{L_{\rm 0}}}\right)^{2/3}. 
    \end{gather}
    
    Note that the ratio $\nicefrac{T_1}{T_0}$ can be replaced by $\nicefrac{D_1}{D_0}$ under fixed TPP. Knowing how $L,\mu,\rho$ change with model size and batch size,\footnote{In real-world applications, the change of constants with a model size might be ignored for simplicity, but later we provide estimates for them that we use in \Cref{sec:experiments}.} we can adjust the batch size and sequence length for a larger model.

    \subsection{Tuning the Frank--Wolfe Stepsize} 
    
    From \eqref{eq:critical_stepsize} we know that the optimal Frank--Wolfe stepsize $\beta$ should scale as $\frac{1}{K}$; therefore, we have 
    \begin{equation}\label{eq:how_to_scale_stepsize}
        \hspace{-3mm}\frac{\beta_{\rm 0}^{\star}}{\beta_{\rm 1}} = \frac{\nicefrac{B_{\rm 0}^{\star}S_{\rm 0}^{\star}}{T_0}}{\nicefrac{B_{\rm 1}S_{\rm 1}}{T_1}}    \Rightarrow 
        \beta_{\rm 1} 
        = \beta_{\rm 0}^{\star} \frac{B_{\rm 1}S_{\rm 1}}{B_{\rm 0}^{\star}S_{\rm 0}^{\star}}\frac{T_0}{T_1}.
    \end{equation}
Since we increase batch size and sequence length according to \eqref{eq:how_to_scale_batch_size}, then the optimal Frank--Wolfe stepsize for a larger model is expected to be around
\begin{equation}
        \beta_{\rm 1} =
        \beta_{\rm 0}^{\star}\left(\frac{\frac{\sqrt{T_{\rm 0}}}{\sqrt{T_{\rm 1}}}  \frac{\mu_{\rm 1}}{\mu_{\rm 0}}   \frac{\rho_{\rm 1}}{\rho_{\rm 0}}}{\frac{L_{\rm 1}}{L_{\rm 0}}}\right)^{2/3}.
    \end{equation}

\subsection{Batch Size Scheduling}\label{sec:batch_size_scheduling_theory}

We now consider a training setting in which data arrives sequentially rather than being fully available upfront.
In this setting, a model is first trained on an initial corpus and subsequently updated as additional data becomes available, causing the effective token budget to grow over time.

This departs from standard pretraining assumptions and raises a practical question: how should hyperparameters such as batch size and sequence length be adapted as the available token budget increases? Na\"{i}vely reusing batch and sequence settings tuned for early stages can lead to suboptimal token efficiency and slower convergence.

In the following, we propose a principled and practically implementable pipeline for selecting and adapting batch size and sequence length in the delayed-data regime.

\paragraph{First stage (training with $T_{(1)}=T_0$ tokens).}
 Assume that in the beginning, we only have a smaller token budget $T_{(1)}=T_0$, which is sufficient to train a smaller model efficiently, but insufficient to do so for a larger model. 
 
The remaining tokens $T_{(2)}=T_1 - T_0$ arrive at a later time. Based on \eqref{eq:how_to_scale_batch_size}, when training the large model using $T_{(1)}$ tokens,\footnote{We should use $T_{(1)}$ instead of $T_1$ in \eqref{eq:how_to_scale_batch_size} when TPP is not fixed.} our theory suggests choosing the batch size $B_1$ and sequence length $S_1$ such that
\begin{equation}\label{eq:batch_size_first_stage}
    B_1S_1 = B_0^{\star}S_0^{\star}\left(\frac{\nicefrac{\mu_1}{\mu_0}\cdot \nicefrac{\rho_1}{\rho_0}}{\nicefrac{L_1}{L_0}}\right)^{2/3} \aaprox{holds if problem constants do not change significantly} B_0^{\star}S_0^{\star},
\end{equation}
where \annotate, that is, the effective batch--sequence scale for the large model should closely match that of the small model when the problem-dependent constants do not vary substantially with the model size (see \Cref{sec:constant_estimation}). 

From \eqref{eq:how_to_scale_stepsize}, the Frank--Wolfe stepsize should be chosen as 
\begin{equation}\label{eq:stepsize_first_stage}
    \beta_1 \aeq{holds since the first stage involves $T_0$ tokens to train a larger model} \beta_0^{\star}\frac{B_1S_1}{B_0^{\star}S_0^{\star}} = \beta_0^{\star}\left(\frac{\nicefrac{\mu_1}{\mu_0}\cdot \nicefrac{\rho_1}{\rho_0}}{\nicefrac{L_1}{L_0}}\right)^{2/3} \aaprox{holds when the problem constants do not change significantly} \beta_0^{\star},
\end{equation}
where \annotate.
Such a choice of the Frank--Wolfe stepsize is also recommended by the $\mu$P literature, which advocates keeping the learning rate fixed when the token budget and batch configuration are unchanged.

\paragraph{Second stage (training with the full budget $T_{(1)}+T_{(2)}$).}
Next, we receive an additional $T_{(2)}$ tokens.
Eq.~\eqref{eq:best_eps} suggests that we should expect the optimization error to improve from order $T_0^{-1/3}$ at the end of the first stage to order $T_1^{-1/3}$ at the end of the second stage.
To realize this improvement in practice, we switch to using \eqref{eq:how_to_scale_batch_size} and \eqref{eq:how_to_scale_stepsize} during the second stage, with the full token budget $T_1 = T_{(1)}+T_{(2)}$.

Overall, this hyperparameter restart strategy for Scion suggests selecting the batch size, sequence length, and Frank--Wolfe stepsize based on the total number of tokens that will ultimately be available to the model.
If additional tokens arrive at later times, the same procedure can be repeated: the batch size and sequence length are increased accordingly, and the Frank--Wolfe stepsize is adjusted based on the final token budget that the larger model will observe.

\subsection{Guidelines for Practitioners }

We summarize all the details on how to adjust the optimizer's parameters under the BST scaling rule below to facilitate its implementation in practice. 

\subsubsection{Hyperparameter Scaling: From Small to Large Models}

In this scenario, the model size changes. Therefore, we need to account for a change of optimization problem constants, such as $L,\mu,\rho$. We summarize the resulting procedure below:

\begin{enumerate}
    \item Obtain optimal values of the batch size $B_0^{\star}$ and sequence length $S_0^{\star}$, Frank--Wolfe stepsize $\beta_0^{\star}$ by tuning a small model, while setting momentum parameter $\alpha$ and radii $\eta$ to default values.

    \item Estimate the problem constants $L_0,\mu_0,\rho_0$ and $L_1, \mu_1, \rho_1$ for small and large models, respectively, based on the fitted power laws \eqref{eq:fitted_laws}. 

    \item Choose batch size $B_1$, sequence length $S_1$, and Frank--Wolfe stepsize $\beta_1$ for larger model using \eqref{eq:how_to_scale_batch_size} and \eqref{eq:how_to_scale_stepsize}, namely
    \begin{equation}
        B_1S_1 = B_0^{\star}S_0^{\star}\left(\frac{\frac{T_1}{T_0} \frac{\mu_1}{\mu_0} \frac{\rho_1}{\rho_0}}{\frac{L_1}{L_0}}\right)^{2/3}, 
        \quad 
        \beta_1 = \beta_{\rm 0}^{\star}\left(\frac{\frac{\sqrt{T_{\rm 0}}}{\sqrt{T_{\rm 1}}}  \frac{\mu_{\rm 1}}{\mu_{\rm 0}}   \frac{\rho_{\rm 1}}{\rho_{\rm 0}}}{\frac{L_{\rm 1}}{L_{\rm 0}}}\right)^{2/3},
    \end{equation}
    while keeping radii $\eta$ and momentum $\alpha$ unchanged. 

    \item Use new parameters to train a larger model (either from the beginning or after processing the token budget used for tuning a smaller model).
\end{enumerate}

\subsubsection{Hyperparameter Scaling: From Small to Large Token Budget}

Now assume that the model size remains the same, but the token budget increases. Therefore, the constants $L$ and $\mu$ remain the same, while we need to account for a change of $\rho$ with batch size.

\begin{enumerate}
    \item Obtain optimal values of the batch size $B_0$ and sequence length $S_0$, Frank--Wolfe stepsize $\beta_0$ by tuning a model for a smaller token budget, while setting momentum parameter $\alpha$ and radii $\eta$ to default values.

    \item Estimate the problem constants $\rho_0$ and $\rho_1$ for small and large token budgets, respectively, based on the fitted power laws \eqref{eq:fitted_laws}. 

    \item Choose batch size $B_1$, sequence length $S_1$, and Frank--Wolfe stepsize $\beta_1$ for a larger token budget $T_1$ using \eqref{eq:how_to_scale_batch_size} and \eqref{eq:how_to_scale_stepsize}, namely
    \begin{equation}
        B_1S_1 = B_0^{\star}S_0^{\star}\left(\frac{T_1}{T_0} \frac{\rho_1}{\rho_0}\right)^{2/3}, 
        \quad 
        \beta_1 = \beta_{\rm 0}^{\star}\left(\frac{\sqrt{T_{\rm 0}}}{\sqrt{T_{\rm 1}}}   \frac{\rho_{\rm 1}}{\rho_{\rm 0}}\right)^{2/3},
    \end{equation}
    while keeping radii $\eta$ and momentum $\alpha$ unchanged. 

    \item Use new parameters to train a model for a longer horizon $T_1$ (either from the beginning or after processing the token budget used for a smaller model).
\end{enumerate}

\section{Experiments}\label{sec:experiments}

In this section, we empirically evaluate our theoretical results by training a modded NanoGPT model on the FineWeb dataset, following the experimental setup of \citet{pethick2025scion} and based on the codebase of \citet{jordan2024moddednanogpt}. Details are given in \Cref{sec:experimental_setup}. For Scion, we adopt the recommended operator norms (Sign $\to$ Spectral $\to$ Sign): we choose the radius $\eta=3000$ for sign-updated layers and $\eta=50$ for matrix-type layers. Concretely, this corresponds to using the polar factor of the gradient for matrix-valued parameters and the elementwise sign of the gradient for all other parameter types (\textit{cf.}, \citep{pethick2025scion}).

\subsection{Verification of Assumption \ref{asmp:bounded_variance}}\label{sec:experiments_variance}

First, we empirically test the validity of \Cref{asmp:bounded_variance} when training a 124M base model with Scion for a fixed number of iterations $K=5100$. 
To approximate the gradient variance as a function of the batch size $B$, we sample $m$ mini-batch gradients of size $B$ such that $mB=32768$, and compute the empirical variance across the sampled $m$ mini-batch gradients. 
We track the evolution of this empirical variance over training in \Cref{fig:variance_evolution} and observe that it stabilizes rapidly after a short initial transient phase. 
In \Cref{fig:variance}, we report the final empirical variance values measured at the end of training. 
The fitted power-law relationships support $\sigma^2 \sim \frac{1}{BS}$ as a reasonable working approximation in the regime $B S \ll T$.

\subsection{Verification of Assumption \ref{asmp:mu_kl}}\label{sec:experiments_mu_kl}

Second, we conduct experiments to assess the validity of \Cref{asmp:mu_kl} in practice. 
We use the same experimental setup as in the previous section and track both the dual norm of mini-batch gradients and the corresponding mini-batch training loss throughout training. 
When using Scion, the primal and dual norms are defined as
\begin{equation}\label{eq:norm_definition}
\|x\| = \max_{\ell\in[N]}\|x_{\ell}\|_{\ell}, \quad 
\|x\|_* = \sum_{\ell=1}^N\|x_{\ell}\|_{*,\ell},
\end{equation}
where $\|x_{\ell}\|_{\ell}$ and $\|x_{\ell}\|_{*,\ell}$ denote the primal and dual norms of the $\ell$-th layer of the network with $N$ layers, respectively. 
Their precise definitions are provided in Table~2 (second and third columns) of \citet{pethick2025scion}. See also the recent work by \citet{crawshaw2025exploration}.

We report the joint evolution of the dual gradient norm and the training loss over the course of training in \Cref{fig:mu_kl_verification}. 
We observe that, once the training loss falls below approximately $5$, the data points closely follow a linear relationship, empirically supporting the use of \Cref{asmp:mu_kl} in this setting. 
To quantify this relationship, we estimate the slope using a robust linear regression model with Huber loss, which interpolates between least squares and absolute-error ($\ell_1$) regression and thereby reduces sensitivity to outliers.

\begin{table}[t]
\caption{Final validation loss when training a 124M NanoGPT model varying the batch size while keeping the validation and train sequence lengths $1024$ under the token budget $1.3$B (TPP 10.8). We report the average across $5$ runs along with a standard deviation. {\bf Bold} numbers indicate the best performance in the column. The runs in \textcolor{darkred}{\bf red} indicate the best configuration of batch size, sequence length, and Frank--Wolfe stepsize across all runs for a given token budget.}

    \centering
    \label{tab:batch_size_abl_124M}
    \resizebox{\linewidth}{!}{
        \begin{tabular}{c|ccccccc}
            \toprule

            $T=1.3$B & \multicolumn{7}{c}{\bf Batch Size} \\ 
            \toprule

            $\boldsymbol{\beta}, \times 10^{-4}$ &
            {\bf 64} &
            {\bf 128} &
            {\bf 256} &
            {\bf 512} &
            {\bf 1024} &
            {\bf 2048} &
            {\bf 4096} \\ 
            \midrule 
            {\bf 1.2} &
            {\bf 3.4258}$_{\pm 0.0004}$ &	
            3.3889$_{\pm 0.0012}$ &
            3.3857$_{\pm 0.0013}$ &	
            3.4074$_{\pm 0.0010}$ &	
            3.4587$_{\pm 0.0010}$&	
            3.5598$_{\pm 0.0012}$ &	
            3.7715$_{\pm 0.0013}$ \\
            {\bf 2.4} &
            3.4394$_{\pm 0.0007}$ &
            {\bf 3.3880}$_{\pm 0.0043}$ &
            \textcolor{darkred}{\bf 3.3706}$_{\pm 0.0019}$ &
            3.3801$_{\pm 0.0016}$ &
            3.4144$_{\pm 0.0004}$ &
            3.4940$_{\pm 0.0009}$ &
            3.6694$_{\pm 0.0021}$\\
            {\bf 3.6} &
            3.4554$_{\pm 0.0008}$&
            3.3945$_{\pm 0.0015}$ &
            \textcolor{darkred}{\bf 3.3717}$_{\pm 0.0020}$ &
            {\bf 3.3765}$_{\pm 0.0017}$ &
            {\bf 3.4065}$_{\pm 0.0007}$ &
            {\bf 3.4799}$_{\pm 0.0017}$ &
            {\bf 3.6472}$_{\pm 0.0023}$ \\
            {\bf 4.8} &
            3.4766$_{\pm 0.0016}$ &
            3.4072$_{\pm 0.0048}$ &
            3.3790$_{\pm 0.0013}$ & 
            3.3807$_{\pm 0.0006}$ &
            3.4115$_{\pm 0.0025}$ &
            3.4945$_{\pm 0.0038}$ &
            3.6611$_{\pm 0.0020}$ \\
            {\bf 6.0} &
            3.4967$_{\pm 0.049}$ &
            3.4198$_{\pm 0.0002}$ &
            3.3875$_{\pm 0.0019}$&
            3.3887$_{\pm 0.0024}$ &
            3.4202$_{\pm 0.0022}$ &
            3.5005$_{\pm 0.0038}$ &
            3.7013$_{\pm 0.0160}$ \\
            {\bf 7.2} &
            3.5151$_{\pm 0.0007}$ &
            3.4301$_{\pm 0.0001}$ &
            3.3978$_{\pm 0.0026}$ &
            3.3960$_{\pm 0.0022}$ & 
            3.4331$_{\pm 0.0025}$ &
            3.5270$_{\pm 0.0071}$ &
            3.7506$_{\pm 0.0271}$
            \\ \bottomrule
        
        \end{tabular}
        } 
\end{table}

\begin{table}[t]
\caption{Final validation loss when training a 124M NanoGPT model varying the train sequence length while keeping the batch size $256$ under the token budget $1.3$B. The validation sequence length is always $1024$. We report the average across $5$ runs along with a standard deviation. ${}^{*}$ indicates that not all runs had a stable decrease in validation loss. {\bf Bold} numbers indicate the best performance in the column. The runs in \textcolor{darkred}{\bf red} indicate the best configuration of batch size, sequence length, and Frank--Wolfe stepsize across all runs for a given token budget.}
    \centering
    \label{tab:seq_len_abl_124M}
    \resizebox{0.75\linewidth}{!}{
        \begin{tabular}{c|ccccc}
            \toprule

            $T=1.3$B & \multicolumn{5}{c}{\bf Sequence Length} \\ 
            \toprule

            $\boldsymbol{\beta}, \times 10^{-4}$ &
            
            {\bf 256} &
            {\bf 512} &
            {\bf 1024} &
            {\bf 2048} &
            {\bf 4096} \\ 
            \midrule 
            {\bf 1.2} &
            {\bf 3.7076}$_{\pm 0.0084}$&
            3.4647$_{\pm 0.0240}$&	
            3.4587$_{\pm 0.0010}$&
            3.4126$_{\pm 0.0014}$&
            3.4811$_{\pm 0.0025}$ \\
            {\bf 2.4} &
            3.9622$_{\pm 0.0585}$ ${}^*$&
            {\bf 3.4633}$_{\pm 0.0091}$&
            \textcolor{darkred}{\bf 3.3706}$_{\pm 0.0019}$&
            3.3834$_{\pm 0.0026}$&
            3.4299$_{\pm 0.0021}$\\
            {\bf 3.6} &
            4.0441$_{\pm 0.2055}$ ${}^*$&
            3.4796$_{\pm 0.0131}$&
            \textcolor{darkred}{\bf 3.3717}$_{\pm 0.0020}$&
            {\bf 3.3792}$_{\pm 0.0022}$&
            {\bf 3.4216}$_{\pm 0.0013}$ \\
            {\bf 4.8} &
            3.9292$_{\pm 0.0852}$ ${}^*$&	
            3.5004$_{\pm 0.0063}$&
            3.3790$_{\pm 0.0013}$&
            3.3829$_{\pm 0.0020}$&
            {\bf 3.4243}$_{\pm 0.0029}$\\
            {\bf 6.0} &
            3.9901$_{\pm 0.0170}$ ${}^*$&
            3.5134$_{\pm 0.0059}$&
            3.3875$_{\pm 0.0019}$&
            3.3910$_{\pm 0.0024}$&
            3.4374$_{\pm 0.0037}$ \\
            {\bf 7.2} &
            3.9819$_{\pm 0.1195}$ ${}^*$&
            3.5269$_{\pm 0.0187}$&
            3.3960$_{\pm 0.0022}$&
            3.3987$_{\pm 0.0029}$&
            3.4537$_{\pm 0.0030}$
            \\ \bottomrule
        
        \end{tabular}
        } 

\end{table}

\subsection{Ablations on Batch Size and Sequence Length}\label{sec:ablations_batch}

We conduct ablation studies by varying the batch size $B$ and sequence length $S$ to identify the optimal Frank--Wolfe stepsize $\beta$ for Scion when training a base 124M model with a fixed validation sequence length $1024$. We report results under a fixed token budget of $1.3$B %
in Tables~\ref{tab:batch_size_abl_124M} and \ref{tab:seq_len_abl_124M}. This corresponds to TPP ratio of $10.8$ %
(approximately $0.5\times$ %
the Chinchilla optimum).

We observe that once the batch size (or sequence length) is sufficiently large, the optimal Frank--Wolfe stepsize stabilizes at $3.6 \cdot 10^{-4}$. 
Moreover, the results indicate that, for the base model, the optimal batch size and sequence length are approximately $256$ and $1024$%
, yielding the lowest validation loss. Additionally, for significantly short train sequence lengths of $256$, most runs were unstable and exhibited high standard deviations since the validation loss is $1024$. We also observe that the best performance between batch sizes $256$ and $512$ differs little, indicating that performance is almost batch-independent. This aligns with \Cref{cor:best_eps}, which shows that there exists a batch-independent regime.

\subsection{Estimating Problem-Dependent Constants}\label{sec:constant_estimation}

In our next experiment, we estimate the problem-dependent constants $L$, $\mu$, and $\rho$ across different model configurations in order to track how these quantities change with model size. 
Specifically, we train models using a fixed Frank--Wolfe stepsize $\beta = 3.6 \cdot 10^{-4}$, batch size $B = 512$, and sequence length $S = 1024$ for $5100$ iterations, following the ablation study in \Cref{sec:ablations_batch}, while varying the number of layers $\texttt{n\_layer}$ and the embedding dimension $\texttt{n\_embd}$. In this section, we ignore the change in the constants $L,\mu,\rho$ with the batch size, but later we account for this dependency in the hyperparameter transfer. The estimated values are reported in \Cref{sec:empirical_verification_training_setup}.
The estimation procedure is carried out as follows.

\paragraph{Smoothness constant $L$.}
    To estimate the smoothness constant, we measure the following ratio
    \[
    \frac{\|g(x_k;\xi_k) - g(x_{k-1};\xi_{k-1})\|_*}{\|x_k - x_{k-1}\|},
    \]
    where $g(x_k;\xi_k)$ and $g(x_{k-1};\xi_{k-1})$ denote the mini-batch gradients at two consecutive iterations, and the norms are defined as in the previous section. 
    This quantity has been used in prior work as a proxy for local curvature during training \citep{alimisis2025we, zhang2019root, riabinin2025gluon}. 
    As a final estimate of $L$, we average the measured ratio over the last $100$ iterations.

\paragraph{KL condition constant $\mu$.} The estimation of $\mu$ follows the same procedure as in \Cref{sec:experiments_mu_kl}. 
    In particular, we fit a robust linear regression model with Huber loss to the relationship between the dual gradient norm and the training loss, and use the resulting slope as an estimate of $\mu$.

\paragraph{Norm-equivalence constant $\rho$.} In the proof of \Cref{thm:str_decay_mu_kl_expectation_no_restarts}, we apply \Cref{asmp:norm_equiv} to bound terms of the form
    $$
    \|g(x_k;\xi_k) - \nabla f(x_k)\|_* \le \rho \, \|g(x_k;\xi_k) - \nabla f(x_k)\|_2.
    $$
    To approximate the full gradient $\nabla f(x_k)$, we follow the same procedure described in \Cref{sec:experiments_variance}. 
    We track the ratio between the dual norm and the Euclidean norm throughout training, and report the average of this ratio over the last $100$ iterations as an estimate of $\rho$.

We conduct the estimation procedure for several model configurations and fit a shifted power law\footnote{The choice of the fitting model is flexible, and alternative functional forms could also be considered. We leave the exploration of other functional dependencies to future work.} for the problem constants $L,\mu,\rho$ of the form:

\begin{align}\label{eq:fitted_laws}
    &\mu(\texttt{n\_layer}, \texttt{n\_embd}) = 5.2(\texttt{n\_layer}+1.7)^{-0.2};\\
    &L(\texttt{n\_layer},\texttt{n\_embd}) = 0.4(\texttt{n\_layer}+0.7)^{0.2}(\texttt{n\_embd}+126)^{0.35};\notag\\
    &\rho(\texttt{n\_layer}, \texttt{n\_embd}, \texttt{batch\_size}) = 4.1(\texttt{n\_layer}-2.7)^{0.25}(\texttt{n\_embd}-250.8)^{0.3}(\texttt{batch\_size}-9.4)^{0.1}.\notag
\end{align}

\begin{table}[t]
    \centering
    \caption{Estimated problem-dependent constants, assuming that they change with the number of layers \texttt{n\_layer}, embedding dimension \texttt{n\_embd}, and batch size \texttt{batch\_size} according to \eqref{eq:fitted_laws}. The estimations of the change for $\beta$ and $BS$ are based on \eqref{eq:how_to_scale_batch_size} and \eqref{eq:how_to_scale_stepsize}.}
    \label{tab:estimated_constants}
    \resizebox{0.8\linewidth}{!}{
        \begin{tabular}{c|ccccc}
            \toprule
            
            {\bf Model} & $\boldsymbol{L}$ & $\boldsymbol{\mu}$ & $\boldsymbol{\rho}$ & \makecellnew{\bf How to change $\boldsymbol{\beta}$ \\ {\bf w.r.t. 124M model?}} & \makecellnew{\bf How to change $\boldsymbol{B}\boldsymbol{S}$ \\ {\bf w.r.t. 124M model?}}\\
            \midrule
            {\bf 124M} & $7.2$ & $3.1$ & $62.7$ & $1\times$ & $1\times$\\

            {\bf 1B} & $10.6$ & $2.9$ & $111.9$ & \makecellnew{$\searrow 0.5\times~{}^{(a)}$ \\ $\searrow 0.54\times~{}^{(b)}$} & \makecellnew{$\nearrow 4\times~{}^{(a)}$ \\ $\nearrow 4.37\times~{}^{(b)}$}

            \\ \bottomrule
        
        \end{tabular}
        } 
         \begin{tablenotes}
      {\scriptsize 
      \item ${}^{(a)}$ Taking into account the practical requirement that $B$ and $S$ should be powers of two. We increase the product $BS$ rounding to the closest power of two.
      \item ${}^{(b)}$ Ignoring the practical requirement that $B$ and $S$ should be powers of two.
        }
        \end{tablenotes}
\end{table}

Interestingly, the constant $\mu$ decreases with \texttt{n\_layer}, while it remains unchanged with \texttt{n\_embd}. In contrast, the constants $L$ and $\rho$ increase with both \texttt{n\_layer} and \texttt{n\_emdb}.  Using the fitted power laws, we estimate the constants for $2$ model configurations used in the experiments of size 124M and 1B.

We observe that the problem-dependent constants vary slowly with model size and remain relatively stable across the model configurations we consider. Although we use these estimates in subsequent experiments, neglecting this variation does not significantly affect the resulting derivations. The impact of these changes becomes more pronounced only in regimes where $D_1\gg D_0$. Using estimated constants, we characterize how the Frank--Wolfe stepsize $\beta$ and the product $BS$ should be set for the 1B model, knowing the optimal configuration ($B=256, S=1024, \beta=3.6\cdot10^{-4}$) for the 124M model in \Cref{tab:estimated_constants} and using \eqref{eq:how_to_scale_batch_size}, \eqref{eq:how_to_scale_stepsize}. Note that we provide two configurations for the 1B model: whether the practical requirement that the batch size and sequence length be powers of $2$ should be taken into account.

\subsection{Increasing Batch Size and Sequence Length during Training}\label{exp:batch_size_scheduling}

Next, we evaluate the proposed strategy from \Cref{sec:batch_size_scheduling_theory}. 
We use the 124M model as a base model, for which we previously identified batch size $B_0=256$, sequence length $S_0=1024$, and Frank--Wolfe stepsize $\beta_0=3.6\cdot10^{-4}$ as providing the best performance under a token budget $T_0 = 1.3$B. We then consider training a larger 1B model under a total token budget of $T_1 = 10.8$B (the same TPP) using the following strategies:

\begin{itemize}
\item \textbf{Restarted Scion.}
    We follow the strategy described in \Cref{sec:batch_size_scheduling_theory}. 
    During the first stage, corresponding to the initial $T_0$ tokens, we use batch size $B_0 = 256$ and sequence length $S_0 = 1024$ with stepsize $\beta_0 = 3.6\cdot10^{-4}$. 
    After processing $T_0$ tokens, we increase the product $BS$ four times and consider two restarted schemes: $B_1 = 512$ and sequence length $S_1 = 2048$ (in \textcolor{color5}{yellow}) and $B_2=1024$ and sequence length $S_2=1024$ (in \textcolor{color7}{gray}), both with stepsize $\beta_1 = \beta_0/2 = 1.8\cdot10^{-4}$ for the remaining token budget,\footnote{The batch size, sequence length, and Frank–Wolfe stepsize used in the first and second training stages are determined using \eqref{eq:how_to_scale_batch_size} and \eqref{eq:how_to_scale_stepsize} and reported in \Cref{tab:estimated_constants}.} following the derivations in \eqref{eq:how_to_scale_batch_size} and \eqref{eq:how_to_scale_stepsize}, with estimates of the problem-dependent constants taken from \Cref{sec:constant_estimation}.

\item \textbf{Fixed \emph{tuned}-batch Scion.}
    We train the 1B model using the \emph{tuned} batch size $B_0 = 256$, which was obtained on a smaller 124M model. We set the sequence length $S_0 = 1024$ with Frank--Wolfe stepsize $\beta_0 = 3.6\cdot10^{-4}$ over the entire horizon $T_1 = 10.8$B (in \textcolor{color0}{light blue}). This configuration is motivated by hyperparameter transfer results under the $\mu$P framework, where the hyperparameters tuned for a smaller model are used when training a larger model.

\item \textbf{Fixed \emph{large}-batch Scion.}
    We train the 1B model using a larger batch size--sequence-length product. In particular, we consider two settings. For $B_1 = 512, S_1 = 2048,$ we evaluate two baselines trained from the beginning over the full token budget $T_1 = 10.8$B: one with Frank--Wolfe stepsize $\beta_0$ (in \textcolor{color1}{orange}), suggested by the $\mu$P framework, and one with stepsize $\beta_1$ (in \textcolor{color3}{pink}), suggested by \eqref{eq:how_to_scale_stepsize}. For $B_2=1024 \quad \text{and} \quad S_2=1024,$ we again train from the beginning over the full token budget $T_1 = 10.8$B, and consider two baselines: one with stepsize $\beta_0$ (in \textcolor{color2}{blue}), suggested by the $\mu$P framework, and one with stepsize $\beta_1$ (in \textcolor{color4}{green}), suggested by \eqref{eq:how_to_scale_stepsize}.

\end{itemize}

\begin{figure}[t]
    \centering
    \begin{tabular}{c}
    \hspace{-3mm}
    \includegraphics[width=\linewidth]{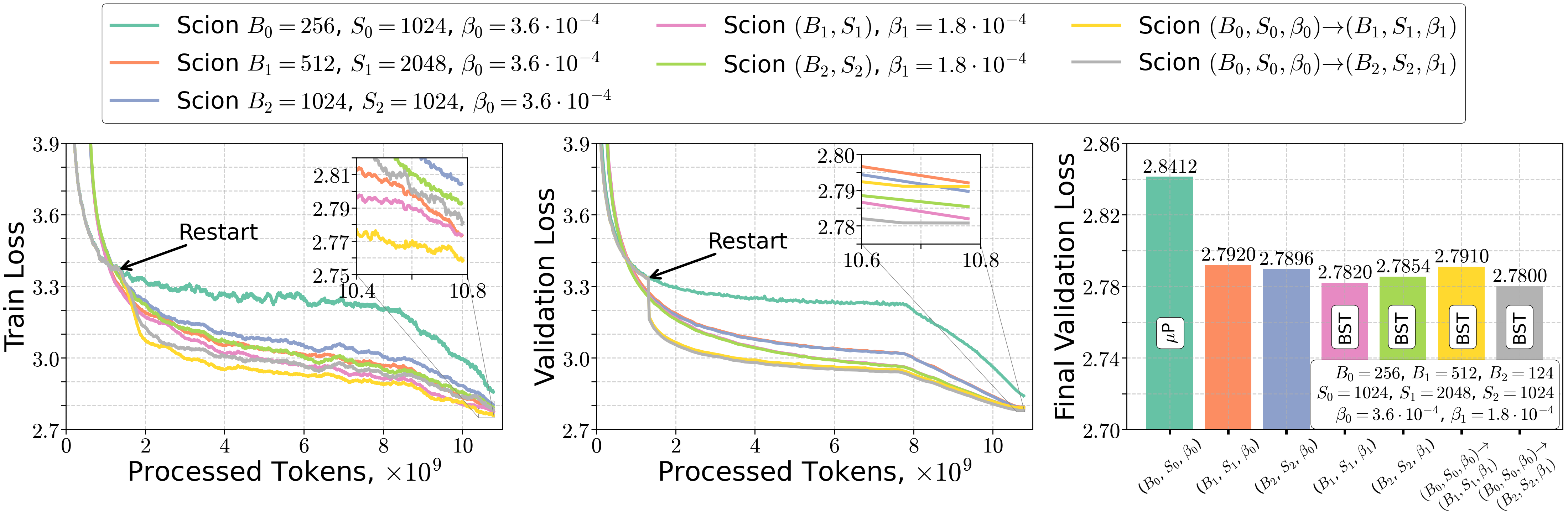} 
    
    \end{tabular}
    \caption{Comparison of batch size and sequence length scheduling strategies when training a 1B model. The restarting schemes (in \textcolor{color5}{yellow} and \textcolor{color7}{gray}) are compared against fixed schedules. The validation loss is evaluated with a smaller sequence length of $1024$. The values of batch sizes $B_{0,1,2}$, sequence lengths $S_{0,1}$, and Frank--Wolfe stepsizes $\beta_{0,1}$ are given in the legends. The notation $(B_{0,1,2}, S_{0,1,2}, \beta_{0,1})$ characterizes which batch size, sequence length, and Frank--Wolfe stepsize are used for the particular setup. The notation $(B_0,S_0,\beta_0)\to(B_{1,2}, S_{1,2},\beta_1)$ characterizes how parameters of Scion change after restart (e.g., batch size increases from $B_0$ to $B_{1,2}$), respectively. The notation $\mu$P or BST indicates the rule used to select $B, S$, and $\beta$.
    }
    \label{fig:restarted_results}
\end{figure}

\begin{remark}
   We note that the choice of batch size $B$ and sequence length $S$ in this set of experiments is partially guided by practical considerations, as these values are typically selected as powers of $2$. We follow this convention to evaluate the performance of the restarted and small- and large-batch baselines in a setting that more closely reflects real-world practice. However, in the later experiments aimed at demonstrating hyperparameter transfer, we select $B$ and $S$ strictly according to the BST scaling rule, ignoring the aforementioned practical constraints.
\end{remark}

Based on the results in \Cref{fig:restarted_results}, we can make the following claims. 
\begin{enumerate}
    \item The $\mu$P framework, where all parameters of the algorithm, i.e., Frank--Wolfe stepsize, batch size, and sequence length remain unchanged, achieves the worst performance. This result demonstrates the limitation of the $\mu$P framework, which ignores changes of batch size and sequence length. Our BST scaling instead suggests increasing the product $BS$ that leads to enhanced performance.

    \begin{findingbox}
    \textbf{Finding 1.} The training configuration suggested by the $\mu$P framework becomes sub-optimal when the batch size and/or sequence length increase.
    \end{findingbox}

    \item Both restarting strategies for Scion demonstrate competitive performance compared to the other baselines. After the restart, both variants show accelerated improvement in training and validation loss relative to the baselines. Moreover, the training curves of the restarted Scion models remain consistently below those of the other methods from the restart point onward.

    \begin{findingbox}
    \textbf{Finding 2.} Restarting strategies improve performance in comparison to other fixed large-batch baselines. Eventually, both restarted and fixed large-batch variants of Scion match in performance at the end of the training. 
    \end{findingbox}
    
    \item In addition, we observe that quadrupling the batch size while keeping the sequence length fixed performs slightly better than doubling both the batch size and the sequence length. In this setting, the former strategy achieves approximately a $0.01$ lower validation loss than the latter.

    \begin{findingbox}
    \textbf{Finding 3.} Increasing batch size while keeping sequence length might be slightly more preferable than increasing both parameters simultaneously if context extension is not necessary. Otherwise, extending the sequence length (i.e., adding this capability) must be compensated for by the batch size for optimization efficiency.  
    \end{findingbox}

    \item All large-batch baselines (with both values of Frank--Wolfe stepsize: $\beta_0$ suggested by $\mu$P and $\beta_1$ suggested by BST rule) achieve similar performance. The best performance is achieved when we double the batch size and the sequence length with Frank--Wolfe stepsize set according to \eqref{eq:how_to_scale_stepsize}. The other three baselines are slightly worse and achieve the validation loss $0.005-0.01$ higher.  

    \begin{findingbox}
    \textbf{Finding 4.} Large batch size $512$ is sub-optimal for a smaller model with TPP $10.6$, but becomes preferable for a larger model with the same TPP, highlighting limitations of the $\mu$P framework. Even if we tune the smaller model with the larger batch for $\mu$P, the step-size configuration is suboptimal for the larger model. This issue is further confounded by the fact that we often do not know the optimal batch size for the larger model since it depends on the token horizon. 
    \end{findingbox}

\end{enumerate}

\subsection{Increasing Batch Size Further Does Not Help}

Next, we investigate whether increasing the product $BS$ by $8$ or $16$ times (when fixing a train sequence length $1024$, this results in batch sizes $2048$ and $4096$) yields additional benefits when training a larger 1B model. All experiments are conducted with fixed training and validation sequence lengths of $1024$. We report results using both Frank–Wolfe stepsizes suggested by the $\mu$P framework and those determined by our BST scaling rule.

The results are shown in \Cref{fig:all_large_batch_results}. We observe that Scion with a batch size of $1024$ outperforms the baselines with batch sizes $2048$ and $4096$. This finding suggests that our BST scaling rule, which prescribes how to scale the product $BS$, provides a reliable practical guideline. Increasing the product $BS$ beyond this recommendation does not yield further performance gains. In particular, the validation loss for Scion with batch size $2048$ worsens by approximately $0.005$-$0.01$, while for batch size $4096$ the degradation is more pronounced, around $0.03$-$0.04$.

\begin{findingbox}
    \textbf{Finding 5.} Increasing the batch size beyond $1024$, suggested by the BST scaling rule, worsens performance gains.
    \end{findingbox}

\begin{figure}[t]
    \centering
    \begin{tabular}{c}
    \hspace{-3mm}
    \includegraphics[width=\linewidth]{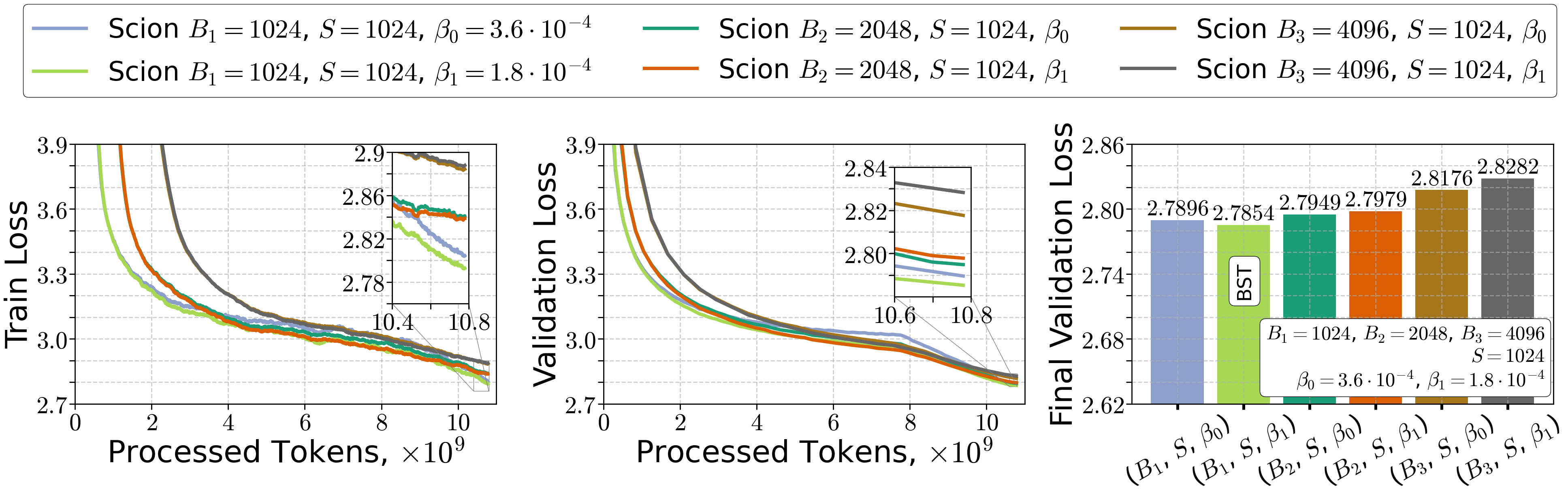} 
    
    \end{tabular}
    \caption{Comparison of fixed large batch size strategies when training a 1B model. The validation loss is evaluated with a smaller sequence length $1024$. Scion with a batch size of $1024$ suggested by our BST scaling rule achieves the best performance compared to other baselines with batch sizes $2048$ and $4096$. The values of batch sizes $B_{1,2,3}$, sequence lengths $S$, and Frank--Wolfe stepsizes $\beta_{0,1}$ are given in the legends. The notation $(B_{1,2,3}, S, \beta_{0,1})$ characterizes which batch size, sequence length, and Frank--Wolfe stepsize are used for the particular setup, respectively. The notation BST indicates the rule used to select the $B, S$, and $\beta$.
    }
    \label{fig:all_large_batch_results}
\end{figure}

\begin{figure}[h]
    \centering
    \begin{tabular}{c}
    \hspace{-3mm}
    \includegraphics[width=\linewidth]{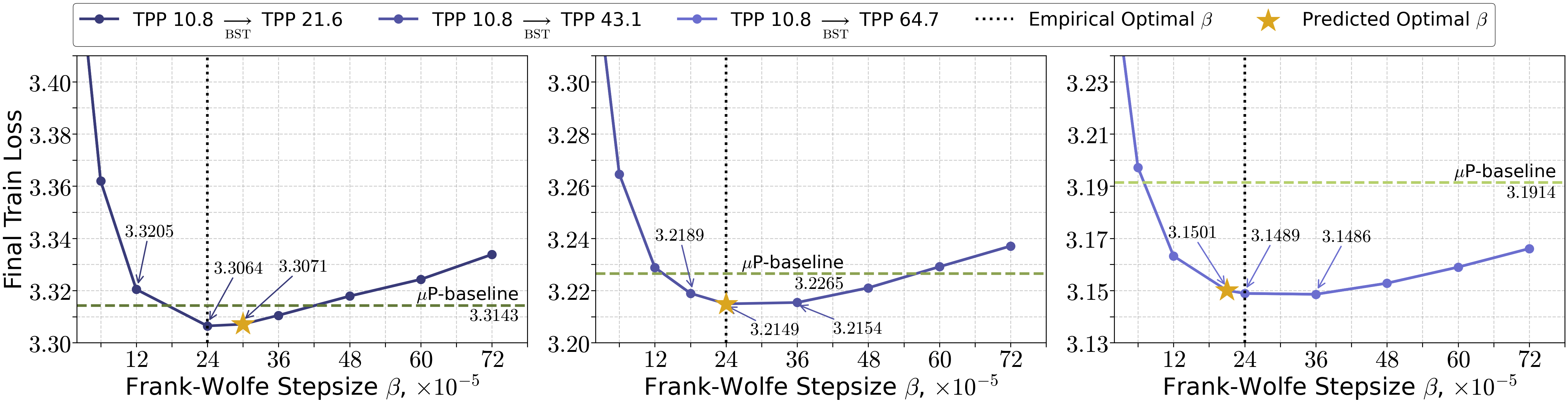} 
    
    \end{tabular}
    \caption{The final performance of the 124M model when varying the Frank--Wolfe stepsize $\beta$ under different token budgets ({\bf left:} 2.7B, {\bf center:} 5.3B, {\bf right:} 8.0B). We average the train loss over 3 random seeds and report the moving average in a window of size 500. We observe that the BST scaling rule predicts a good estimate for the optimal $\beta$ when increasing the token budget. Moreover, the difference in performance between BST and $\mu$P baselines grows with a token budget.
    }
    \label{fig:lr_transfer_124M}
\end{figure}

\begin{figure}[h]
    \centering
    \begin{tabular}{c}
    \hspace{-3mm}
    \includegraphics[width=\linewidth]{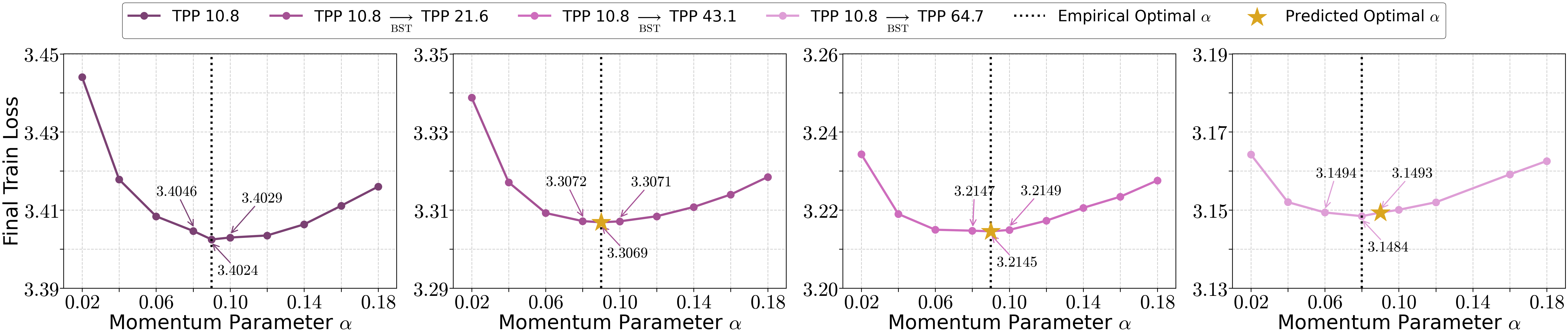} 
    
    \end{tabular}
    \caption{The final performance of the 124M model when varying the momentum $\alpha$ under different token budgets ({\bf left:} 1.3B, {\bf center left:} 2.7B, {\bf center right:} 5.3B, {\bf right:} 8.0B). We average the train loss over 3 random seeds and report the moving average in the window of size 500. We observe that rule momentum parameter $\alpha$ transfers under BST scaling.
    }
    \label{fig:alpha_transfer_124M}
\end{figure}

\begin{figure}[h]
    \centering
    \begin{tabular}{c}
    \hspace{-3mm}
    \includegraphics[width=\linewidth]{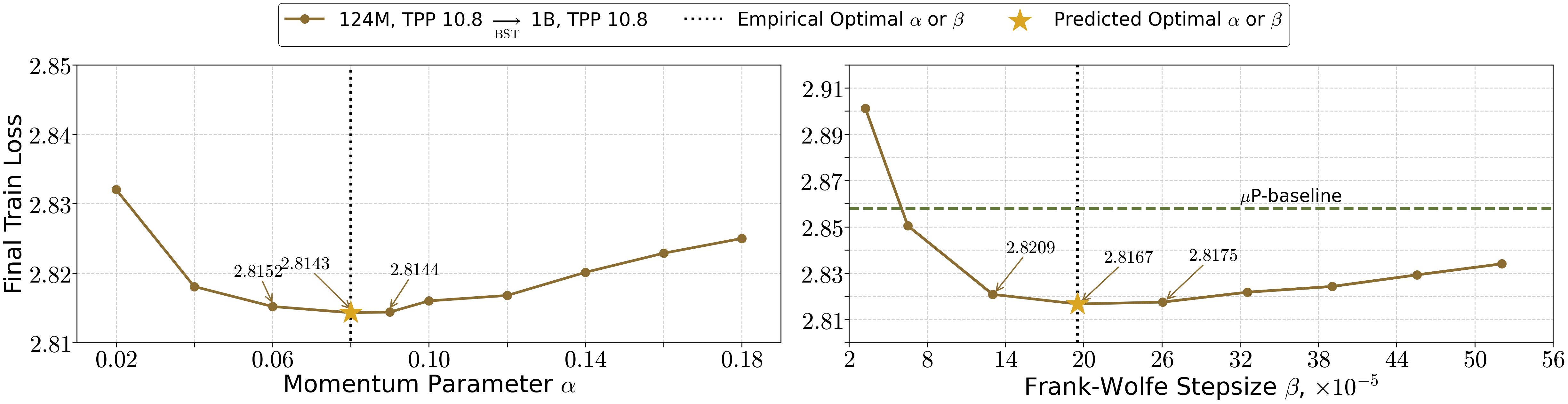} 
    
    \end{tabular}
    \caption{The final performance of the 1B model when varying the momentum parameter $\alpha$ ({\bf left}) and Frank--Wolfe stepsize $\beta$ ({\bf right}) under different token budget 10.6 TPP. We report the final train loss, smoothed in the window of size $500$. We observe that the BST scaling rule predicts a good estimate for both optimal $\alpha$ and $\beta$ when transferring from a smaller 124M model to a larger 1B model.
    }
    \label{fig:alpha_lr_transfer}
\end{figure}

\subsection{Hyperparameter Transfer}

From \Cref{tab:batch_size_abl_124M}, we know that for a base 124M model, the optimal set of hyperparameters is $B=256, S=1024, \beta=3.6\cdot10^{-4}$ under token budget $T=1.3$B (TPP 10.8). We want to use these parameters in BST rule to obtain them for a larger training horizon or model size. In this section, the reported train losses are averaged over 3 random seeds (only for 124M) and smoothed using running average in the window of size 500 (for both 124M and 1B models). We ignore the requirement for $B$ to be the powers of $2$ in these set of experiments.

\subsubsection{Increasing Token Budget for 124M Model}

In this section, we report the pretraining results for 124M model under increased token budgets $(i)$ $T=2.7$B (TPP 21.6), $(ii)$ $T=5.3$B (TPP 43.1), and $(iii)$ $T=8.0$B (TPP 64.7). We set a batch size for longer horizons using \eqref{eq:how_to_scale_batch_size}: $B=416$ for $T=2.7$B, $B=672$ for $T=5.3$B, and $B=896$ for $T=8.0$B, using estimates of problem-dependent constants from \eqref{tab:estimated_constants}. Momentum and sequence length are set to $\alpha=0.1$ and $S=1024$, respectively.

To demonstrate the predictive power of the BST scaling rule in finding optimal Frank--Wolfe stepsize $\beta$ in \eqref{eq:how_to_scale_stepsize}, we report the final train losses when varying $\beta$ under three token budgets. Momentum parameter is fixed to $\alpha=0.1$. We expect the optimal Frank--Wolfe stepsize to be around $(i)$ $\beta=3.0\cdot10^{-4}$, $(ii)$ $\beta=2.4\cdot10^{-4}$, and $(iii)$ $\beta=2.1\cdot10^{-4}$. In \Cref{fig:lr_transfer_124M}, we observe that the BST scaling rule predicts Frank--Wolfe stepsize close to the optimal one in all the cases. Moreover, we observe that $\mu$P baseline $(B=256,S=1024,\beta=3.6\cdot10^{-4})$ becomes more suboptimal when increasing the token budget, which demonstrates the limitations of the $\mu$P framework even further.

Next, we switch to testing the BST rule for predicting the optimal value of the momentum parameter $\alpha$. According to the BST rule, $\alpha$ should transfer. Empirical results in \Cref{fig:alpha_transfer_124M} support this claim. For all values of the token budget, the optimal $\alpha$ is close to $0.09$.

\subsubsection{Increasing Model Size}

Now we want to train a 1B model with batch size $B=1120, S=1024$ under token budget $10.8$B (TPP 10.8). In this setup, we test the predictive power of the BST scaling rule when we change the model size. The value of the batch size is set according to \eqref{eq:how_to_scale_batch_size}, using estimates from \Cref{tab:estimated_constants}. We expect the optimal Frank--Wolfe stepsize to be close to $1.95\cdot10^{-4}$, while the momentum parameter to be close to $0.09$. We report the results in \Cref{fig:alpha_lr_transfer}. We observe that the BST rule provides a good estimation for both the optimal momentum $\alpha$ and Frank--Wolfe stepsize $\beta$, when increasing the model size.

\section{Conclusion}
We developed a token-budget--aware theory for scaling batch size, sequence length, and Frank--Wolfe stepsize in SCG methods under a $\mu$-KL condition.
Our analysis reveals a non-monotone dependence of optimization error on the effective batch--sequence scale and yields a principled BST-scaling rule that identifies when increasing batch size is beneficial and when it becomes suboptimal.
In contrast to hyperparameter transfer approaches that ensure local stability at initialization, our results characterize long-horizon, trajectory-level behavior and explain how hyperparameters should adapt as the token budget grows.
Empirically, we show that large batches are not inherently detrimental: when scaled according to our theory, jointly adapting $(B,S,\beta)$ improves both token efficiency and convergence in large-scale training.

\section*{Limitations and Future Work}

Note that in our experiments, the remaining hyperparameters, such as the radii $\eta$ and the variance initialization, were adopted directly from \citet{pethick2025scion} without additional tuning. This configuration already yields strong empirical performance for Scion. However, for other model architectures, these hyperparameters may not be readily available. In such cases, we recommend selecting them based on prior literature or performing a small hyperparameter sweep. We expect their precise choice to be less critical for final performance than that of the batch size or the Frank–Wolfe stepsize.

More generally, substantially suboptimal choices of these hyperparameters may affect the predictive accuracy of our BST scaling rule. Determining the minimal set of hyperparameters that must be tuned on a small model to ensure that our theoretical predictions remain practically actionable remains an important open question, which we leave for future work.

Another important question that requires further investigation is the transfer of momentum parameter. In our experiments, we observe that increasing the model size slightly shifts the range of near-optimal values of the momentum parameter $\alpha$ toward lower values ($0.06$–$0.08$), compared to the range ($0.08$–$0.1$) for the base 124M model. This shift may stem from power-law fits being based on an insufficient number of data points, a suboptimal functional dependency, or the fact that additional training parameters, such as the token budget, should be incorporated into the scaling rules. Moreover, in the 124M model, most of the parameter mass is concentrated in the embedding and head layers, whereas larger models allocate a greater share of parameters to other layers. This observation suggests that the constants $\rho$, $L$, and $\mu$ may exhibit distinct scaling regimes depending on the model size, which should be taken into account when considering the power laws.

\section*{Acknowledgement}

Rustem Islamov and Aurelien Lucchi acknowledge the financial support of the Swiss National Science Foundation, SNSF grant No 207392. Volkan Cevher acknowledges the financial support of the Swiss National Science Foundation, SNSF grant No 240094. This work was also supported under project ID \# 37 as part of the Swiss AI Initiative, through a grant from the ETH Domain and computational resources provided by the Swiss National Supercomputing Centre (CSCS) under the Alps infrastructure.

\bibliography{refs}
\bibliographystyle{plainnat}

\newpage
\appendix
\counterwithin{figure}{section}
\counterwithin{table}{section}

\vbox{
  {\hrule height 2pt \vskip 0.15in \vskip -\parskip}
  \centering
  {\LARGE\bf Appendix\par}
  {\vskip 0.2in \vskip -\parskip \hrule height 0.5pt \vskip 0.09in}
}

\newcommand\invisiblepart[1]{%
  \refstepcounter{part}%
  \addcontentsline{toc}{part}{\protect\numberline{\thepart}#1}%
}

\invisiblepart{Appendix}
\setcounter{tocdepth}{2}
\localtableofcontents

\appendixtrue

\section{Description of the Experimental Setup}\label{sec:experimental_setup}

Our implementation uses Scaled ReLU${}^2$ from  \citet{large2024scalable} (see Appendix B.2), rotary embeddings \citep{su2021roformer} in place of positional embeddings, RMSNorm \citep{zhang2019root} (without learnable parameters following \citet{pethick2025scion}) instead of LayerNorm, and a linear learning rate decay schedule instead of cosine annealing. The choice of radius is taken from \citet{pethick2025scion}: $\eta=50$ for matrix-type layers and $\eta=3000$ for the rest of the layers. To approximate the polar factor of the gradient, we use the Newton-Schulz method with $5$ iterations, following \citep{jordan2024muon}. All other details are reported in \Cref{tab:training_details}.

Note that due to limited GPU availability, the 1B model is trained using checkpointing. This introduces slight fluctuations across runs. Although the random seed is fixed, some variability remains due to nondeterminism in the PyTorch implementation.

\section{Empirical Verification of Assumptions}\label{sec:empirical_verification_training_setup}

\subsection{Verification of \Cref{asmp:bounded_variance}}\label{apx:verification_subgaussian}

In this section, we provide the evolution of the empirical variance throughout the training when varying the batch size and sequence length when training a base 124M model. When we vary the batch size, we keep the sequence length equal $1024$; when we vary the sequence length, we keep the batch size equal $512$. To approximate the full gradient, we sample a mini-batch gradient of size $32768$. In \Cref{fig:variance_evolution}, we demonstrate that after a short initial phase (up to 1k iterations) the empirical variance stabilizes and fluctuates around the average, suggesting that the variance is fixed during most of the training. 

\begin{table}[t]
    \centering
    \caption{The model configurations and training details used in \Cref{sec:experiments}.}
    \label{tab:training_details}
    \resizebox{0.6\linewidth}{!}{
        \begin{tabular}{c|ccc}
            \toprule

            {\bf Hyperparameter} &
            {\bf 124M Model} &
            {\bf 775M Model} &
            {\bf 1B Model} \\
            \toprule

            {\bf Layers} &
            12 &
            36 &
            18 \\

            {\bf Heads} &
            6 &
            20 &
            16 \\
            {\bf Embedding Size} &

            768 &
            1280 &
            2048  \\
            {\bf Weight Tying} &
             \multicolumn{3}{c}{Yes} \\
            {\bf Activation Function} &
             \multicolumn{3}{c}{ReLU${}^2$} 
             \\
             {\bf Vocabulary Size} & \multicolumn{3}{c}{50304} \\
            {\bf Dataset} & \multicolumn{3}{c}{FineWeb} \\
            {\bf Rampdown} & \multicolumn{3}{c}{28\% of the total token budget} \\
            {\bf Stepsize Schedule} &
            \multicolumn{3}{c}{$\beta_k = \begin{cases}
                \gamma & \text{if } k < n - m\\
                \gamma \cdot \frac{n-k}{m} & \text{otherwise}
            \end{cases}$} \\
            {\bf Gradient Clipping} &
             \multicolumn{3}{c}{ No} \\
            {\bf Momentum Parameter} &
             \multicolumn{3}{c}{$\alpha=0.1$} \\

            {\bf \texttt{lm\_head}/\texttt{embd} Radii} &
             \multicolumn{3}{c}{$3000$} \\
              {\bf Matrix Weights Radii} &
             \multicolumn{3}{c}{$50$} \\
             {\bf Precision} &
             \multicolumn{3}{c}{bf16} \\
             {\bf Device Batch Size} & \makecellnew{32 \\ {\it if the opposite} \\ {\it is not stated explicitly}} & & 16\\

            \bottomrule
        
        \end{tabular}
        } 
\end{table}

\begin{figure}
    \centering
    \begin{tabular}{cc}
    \hspace{-3mm}\includegraphics[width=0.4\linewidth]{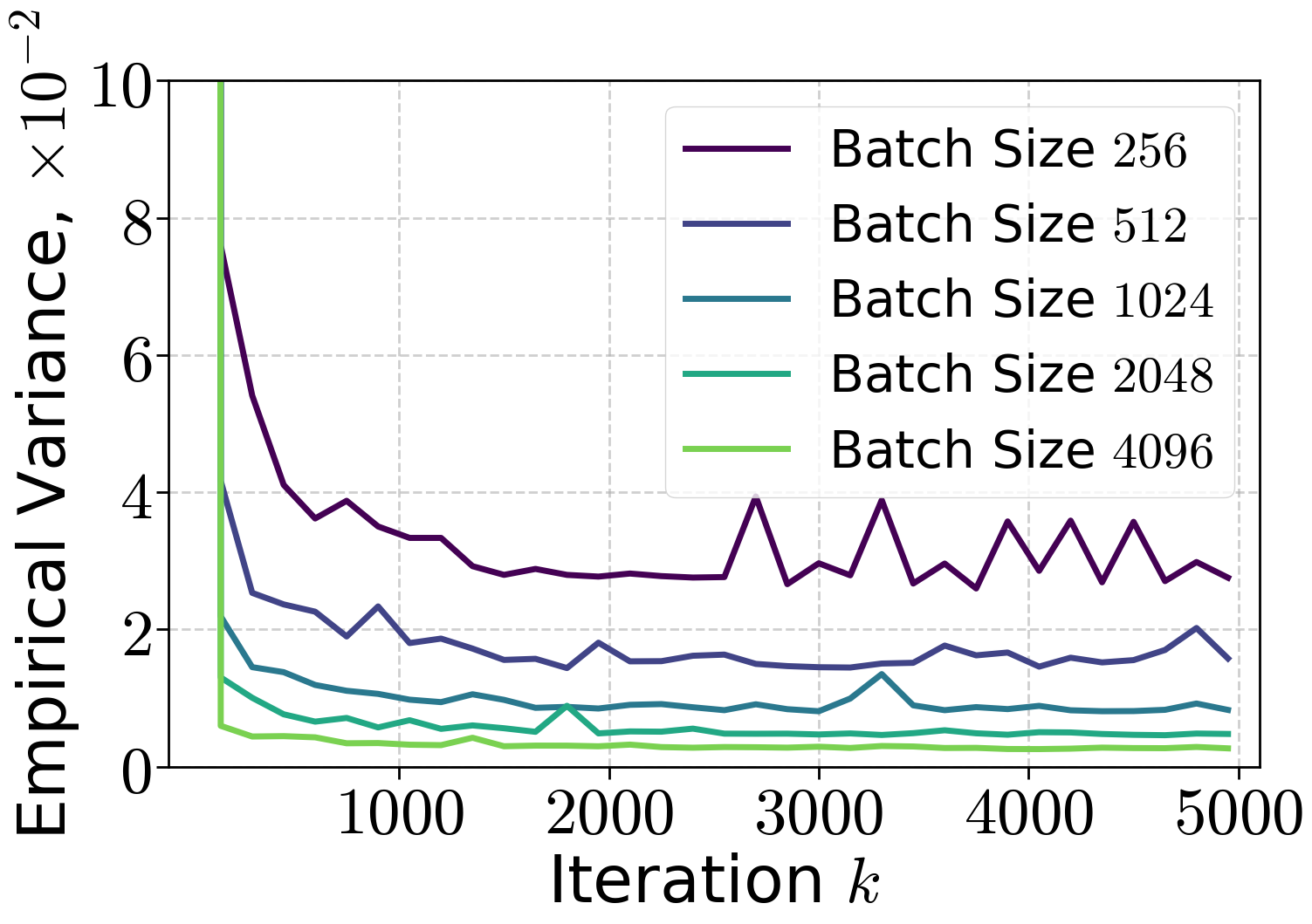} & 
    \hspace{-3mm}
    \includegraphics[width=0.4\linewidth]{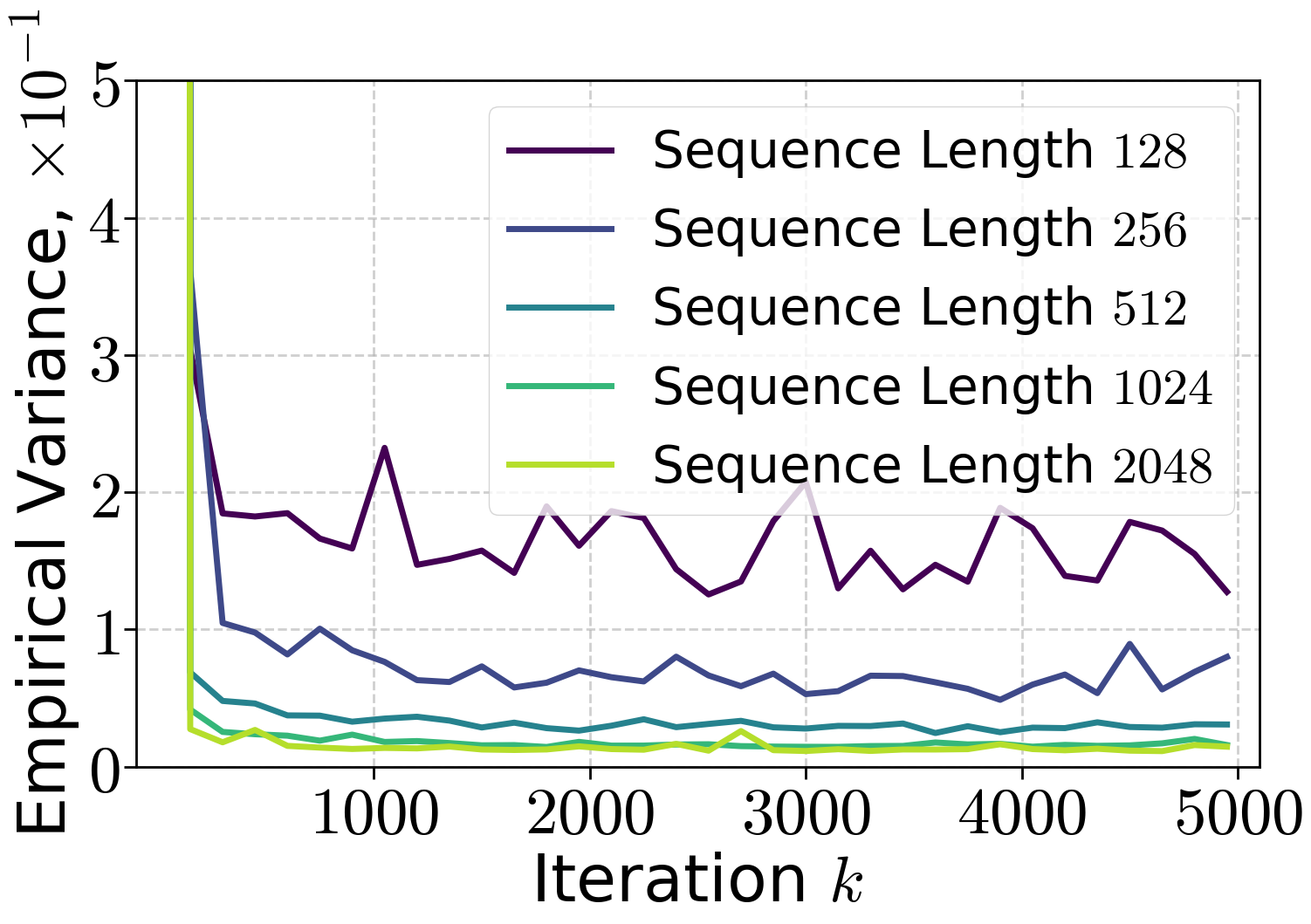}
    \end{tabular}
    \caption{Evolution of the empirical gradient variance varying  batch size $B$ with fixed sequence length $S=1024$ ({\bf left}) and sequence length $S$ with fixed batch size $B=512$ ({\bf right}) when training a 124M NanoGPT model on the FineWeb dataset. We observe that the variance quickly stabilizes after a short initial phase.}
    \label{fig:variance_evolution}
\end{figure}

\subsection{Verification of Assumptions~\ref{asmp:smoothness}-\ref{asmp:mu_kl} when Varying Model Configuration}\label{apx:verification_mu_kl}

In this section, we provide the estimations of problem-dependent constants $L,\mu,\rho$, varying \texttt{n\_embd} and \texttt{n\_layer}, while keeping \texttt{n\_head}=6. We measure the constants for several configurations with a goal to cover a broader number of configurations. Due to extensive requirements on memory and time resources, we do not provide the measurements for all possible configurations. 

\subsubsection{Estimating the Smoothness Constant $L$}

First, we measure the smoothness constant $L$ using the following estimation 
$$\frac{\|g(x_k;\xi_k) - g(x_{k-1};\xi_{k-1})\|_*}{\|x_{k}-x_{k-1}\|},$$
where $g(x_k;\xi_k), g(x_{k-1};\xi_{k-1})$ are mini-batch gradient at two consecutive iterations, while the norms are defined in \eqref{eq:norm_definition}. Based on the results in \Cref{tab:constant_esimations_L}, we fit a power law with shifts of the form
$$
L(\texttt{n\_layer},\texttt{n\_embd}) = C(\texttt{n\_layer}+a_0)^{\nu}(\texttt{n\_embd}+b_0)^{\gamma}.
$$
We fit the parameters of the power law in log-space, using least squares with \texttt{soft\_l1} loss from \texttt{scipy.optimize} \citep{virtanen2020scipy}. The fit provides the following approximations for the constants of the power law:
$$
C=0.4, \quad \nu = 0.2,\quad \gamma=0.35, \quad a_0=0.7, \quad b_0=126.
$$

\begin{table}[t]
    \centering
    \caption{Estimated $L$ constant for various model configurations.}
    \label{tab:constant_esimations_L}
    \resizebox{0.7\linewidth}{!}{
        \begin{tabular}{c|cccccccccc}
            \toprule

            {\bf \texttt{n\_embd}} $\,\backslash\,$ {\bf \texttt{n\_layer}} & 
                {\bf 3} & {\bf 6} & {\bf 9} & {\bf 12} & {\bf 15} & {\bf 18} & {\bf 21} & {\bf 24} & {\bf 27} & {\bf 30} \\
            \midrule
            {\bf 384}   & -- & 5.3 & 6.2 & 6.5 & 6.2 & 18 & 7.6 & -- & 27 & 7.84 \\
            {\bf 576}   & -- & 6.4 & 7.7 & 7.0 & 7.4 & -- & 7.5 & 8.5 & 7.9 & 8.2 \\
            {\bf 768}   & 6.1 & 6.5 & 6.9 & 7.6 & 8.3 & 18 & 9.9 & 8.5 & 8.7 & 10.0 \\
            {\bf 1152}   & -- & -- & -- & 8.8 & 9.4 & 10.8 & 9.5 & -- & -- & -- \\
            {\bf 1536}  & 3 & 7.9 & 9.2 & 12 & 9.9 & -- & -- & -- & -- & -- \\
            {\bf 2304}  & -- & 9.9 & -- & -- & -- & -- & -- & -- & -- & -- \\
            \bottomrule

        \end{tabular}
        } 
\end{table}

\subsubsection{Estimating the Norm Equivalence Constant $\rho$}

\begin{table}[t]
    \centering
    \caption{Estimated $\rho$ constant for various model configurations.}
    \label{tab:constant_esimations_rho}
    \resizebox{0.7\linewidth}{!}{
        \begin{tabular}{c|cccccccccc}
            \toprule

            {\bf \texttt{n\_embd}} $\,\backslash\,$ {\bf \texttt{n\_layer}} & 
                {\bf 3} & {\bf 6} & {\bf 9} & {\bf 12} & {\bf 15} & {\bf 18} & {\bf 21} & {\bf 24} & {\bf 27} & {\bf 30} \\
            \midrule
            {\bf 384}   & -- & 35.5 & 41.3 & 48.2 & 48.7 & -- & 50.7 & -- & -- & 58.0 \\
            {\bf 576}   & -- & 42.1 & 53.8 & 61.1 & 62.9 & -- & 64.2 & 64.6 & 66.2 & 68.6 \\
            {\bf 768}   & 31.2 & 52.6 & 64.1 & 67.1 & 72.4 & -- & 80.2 & 87.0 & 86.3 & 89.2 \\
            {\bf 1152}   & -- & -- & -- & 76.6 & 81.1 & 87.3 & --- & -- & -- & -- \\
            {\bf 1536}  & -- & 67.5 & 77.4 & -- & -- & -- & -- & -- & -- & -- \\
            {\bf 2304}  & -- & -- & -- & -- & -- & -- & -- & -- & -- & -- \\
            \bottomrule

        \end{tabular}
        } 
\end{table}

Now we measure the norm equivalence constant $\rho$. We observed that the $\rho$ constant changes not only with the model size but also with the batch size and sequence length. To measure it, we run Scion with batch size $512$ and sequence length $1024$, and the Frank--Wolfe stepsize $\beta=3.6\cdot10^{-4}$. We estimate the $\rho$ constant as follows
$$
\frac{\|g(x_k;\xi_k) - G(x_k;\Xi_k)\|_*}{\|g(x_k;\xi_k) - G(x_k;\Xi_k)\|_2},
$$
where $g(x_k;\xi_k)$ is a mini-batch gradient of size $512$, while $G(x_k;\Xi_k)$ is a mini-batch gradient of size $8192$, which serves as an approximation of the full gradient. 

We also observe that the constant $\rho$ significantly changes with the batch size. Therefore, we measured how $\rho$ changes with batch size for a configuration with $6$ layers and $768$ embedding dimension in \Cref{tab:rho_batch_size}.

Based on the results in \Cref{tab:constant_esimations_rho} and \Cref{tab:rho_batch_size}, we fit the parameters of the power law of the form
$$
\rho(\texttt{n\_layer},\texttt{n\_embd}, \texttt{batch\_size}) = C(\texttt{n\_layer}+a_0)^{\nu}(\texttt{n\_embd}+b_0)^{\gamma}(\texttt{batch\_size}+c_0)^{\delta}
$$
in log-space, using least squares with \texttt{soft\_l1} loss from \texttt{scipy.optimize}. The fit provides the following approximations for the constants of the power law:
$$
C = 4.1, \quad a_0= -2.7, \quad \nu = 0.25, \quad b_0 = -250.8, \quad \gamma=0.3, \quad c_0=-9.4, \quad \delta = 0.1.
$$

\begin{table}[t]
    \centering
    \caption{Estimated $\rho$ constant for a configuration with $6$ layers and $768$ embedding dimension when varying the batch size.}
    \label{tab:rho_batch_size}
    \resizebox{0.5\linewidth}{!}{
        \begin{tabular}{c|ccccc}
            \toprule

            {\bf \texttt{batch\_size}} & {\bf 256} & {\bf 512} & {\bf 1024} & {\bf 2048} & {\bf 4096} 
            \\
            \midrule
            $\boldsymbol{\rho}$ &  48.9 & 52.6 & 55.0 & 56.4 & 57.9
            \\
            \bottomrule

        \end{tabular}
        } 
\end{table}

\subsubsection{Estimating Kurdyka–Łojasiewicz Constant $\boldsymbol{\mu}$}

\begin{table}[t]
    \centering
    \caption{Estimated Kurdyka–Łojasiewicz constant $\mu$ for various model configurations.}
    \label{tab:constant_esimations_mu}
    \resizebox{0.7\linewidth}{!}{
        \begin{tabular}{c|cccccccccc}
            \toprule

            {\bf \texttt{n\_embd}} $\,\backslash\,$ {\bf \texttt{n\_layer}} & 
                {\bf 3} & {\bf 6} & {\bf 9} & {\bf 12} & {\bf 15} & {\bf 18} & {\bf 21} & {\bf 24} & {\bf 27} & {\bf 30} \\
            \midrule
            {\bf 384}   & -- & 3.4 & 3.1 & 3.1 & 3.0 & -- & 2.9 & -- & -- & 2.7 \\
            {\bf 576}   & -- & 3.3 & 3.2 & 3.0 & 2.9 & -- & 2.7 & 2.6 & 2.5 & 2.4 \\
            {\bf 768}   & 3.7 & 3.2 & 3.0 & 2.9 & 2.8 & -- & 2.6 & 2.5 & 2.3 & 2.4 \\
            {\bf 1152}   & -- & -- & -- & 2.7 & 2.7 & 2.8 & 2.6 & -- & 2.5 & -- \\
            {\bf 1536}  & -- & 3.2 & 2.9 & -- & 2.9 & 3.0 & -- & -- & -- & -- \\
            {\bf 2304}  & -- & 3.6 & -- & -- & -- & -- & -- & -- & -- & -- \\
            \bottomrule

        \end{tabular}
        } 
\end{table}
Finally, we measure the KL constant $\mu$ by tracking the dual gradient norm and train loss (norms are defined in \eqref{eq:norm_definition}). Then, we fit a linear regression with Huber loss, robust to outliers. The slope of the linear fit serves as an approximation of $\mu$ constant. Based on the results in \Cref{tab:constant_esimations_mu}, we fit the parameters of the power law of the form
$$
\mu(\texttt{n\_layer},\texttt{n\_embd}) = C(\texttt{n\_layer}+a_0)^{\nu}(\texttt{n\_embd}+b_0)^{\gamma}.
$$
in log-space, using least squares with \texttt{soft\_l1} loss from \texttt{scipy.optimize}. The fit provides the following approximations for the constants of the power law:
$$
C=5.2, \quad \nu = 0.2,\quad \gamma=0, \quad a_0=1.7, \quad b_0=-384.
$$

\section{Additional Experiments}

\subsection{Additional Baselines in Experiments from \Cref{exp:batch_size_scheduling}}

In this section, we add additional baselines to the setting from \Cref{exp:batch_size_scheduling}. The idea behind the two new baselines is the following. The literature on the learning theory suggests that the excess risk should decay as $\sim \frac{1}{\sqrt{T}}$ under standard convexity \citep{shalev2009stochastic,liu2024new}, which is the closest setting to $\mu$-KL case due to the relation between $\mu$-KL condition and $\zeta$-quasar convexity, described after \Cref{asmp:mu_kl}. This hypothesizes that we need to keep the optimization error close to the excess risk. In particular, if we find that for a small model the dominating term in \eqref{eq:best_eps} is the first one, then we control it as follows
$$
\frac{LB_0S_0}{\mu^2T_0} \sim \frac{\varepsilon_0}{\sqrt{T_0}}.
$$
We want to choose parameters $B_1$ and $S_1$ such that the same approximation holds for a larger model. This gives another recipe on how to increase the batch size and sequence length:
\begin{align}\label{eq:how_to_schedule_batch_size_LT}
    \frac{B_0S_0/T_0}{B_1S_1/T_1} = \frac{1/\sqrt{T_0}}{1/\sqrt{T_1}} \Rightarrow B_1S_1 = B_0S_0 \frac{\sqrt{T_1}}{\sqrt{T_0}}.
\end{align}
From \eqref{eq:how_to_scale_stepsize}, we obtain that we should the Frank--Wolfe stepsize of the form
\begin{equation}\label{eq:how_to_schedule_stepsize_LT}
    \beta_1 = \beta_0\frac{B_1S_1}{B_0S_0}\frac{T_0}{T_1} = \frac{\sqrt{T_0}}{\sqrt{T_1}}.
\end{equation}
For a 1B model, this means that we should increase the product $BS$ by a factor $\sqrt{8}$, while decreasing the Frank--Wolfe stepsize by a factor $1/\sqrt{8}$. Performing all derivations, this gives the values of batch size $736$, sequence length $1024,$ and the Frank--Wolfe stepsize $1.27\cdot10^{-4}$. In \Cref{fig:restarted_results_additional}, we add two more baselines: one when we do a restart with the parameters $736$, $1024$, $1.27\cdot10^{-4}$ after a 1.3B token budget, and another one where we use the parameters $736$, $1024$, $1.27\cdot10^{-4}$ from the beginning. These ideas are closely aligned with prior work \citet{compagnoni2025adaptive,compagnoni2025unbiased,mlodozeniec2025completed}, which also propose to rescale the weight decay (equivalent of our Frank--Wolfe stepsize) following the square-root rule.

We observe that the new baselines are also competitive in practice. However, more aggressive BST scaling rules in \eqref{eq:how_to_scale_batch_size} and \eqref{eq:how_to_scale_stepsize} provide slightly better results: the restarted version of BST baseline achieves the best validation loss, while a fixed batch BST baseline slightly outperforms square-root fixed batch baseline. This further supports that our theory-inspired BST scaling rule is predictive and can be used in practice. The notation Sqrt or BST indicates the rule used to select the Frank--Wolfe stepsize.

\begin{figure}
    \centering
    \begin{tabular}{c}
    \includegraphics[width=\linewidth]{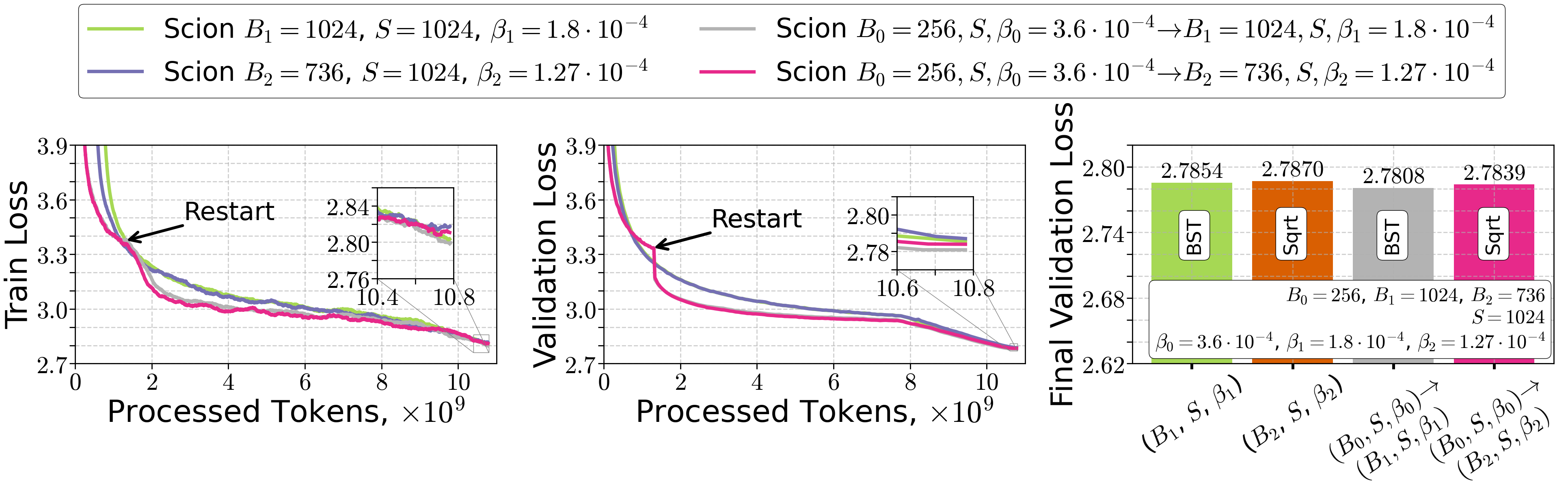} 
    
    \end{tabular}
    \caption{Comparison of two strategies: BST scaling rule, where $BS \sim T^{2/3}$ (fixed $1024$ batch size with $\beta_1=1.8\cdot10^{-4}$ in \textcolor{color4}{light green} and restarted version in \textcolor{color7}{gray}) and square-root rule \eqref{eq:how_to_schedule_batch_size_LT}, where $BS \sim T^{1/2}$ (fixed $736$ batch size with $\beta_2=1.27\cdot10^{-4}$ in \textcolor{newcolor1}{orange} and restarted version in \textcolor{newcolor3}{pink}). The validation and train sequence lengths are fixed to $1024$. ( large batch size strategies when training a 1B model. Scion with a batch size of $1024$ (fixed from the beginning or after a restart) achieves slightly better performance compared to the baselines, where the batch size is set according to the square-root rule \eqref{eq:how_to_schedule_batch_size_LT}. The notation $(B_{1,2}, S, \beta_{0,1})$ characterizes which batch size, sequence length, and Frank--Wolfe stepsize are used for the particular setup, respectively. The notation $(B_0,S,\beta_0)\to(B_{1,2}, S,\beta_{1,2})$ characterizes how parameters of Scion change after restart (e.g., batch size increases from $B_0$ to $B_{1,2}$), respectively.
    }
    \label{fig:restarted_results_additional}
\end{figure}

\subsection{Effect of the Device Batch Size}

In this section, we evaluate how the device batch size affects the final performance of the 124M model. Note that using a smaller device batch size within a fixed global batch size results in more gradient accumulation steps. In \Cref{fig:quantization}, we report the model's final validation loss as we vary the device batch size and the momentum parameter. We observe that for device batch sizes $8$ and $32$, the performance is closely aligned. The quantization errors become slightly visible for a device batch size of $2$ when using extreme values of momentum far from the optimal value. However, around the optimal momentum parameter, the difference is within one standard deviation.

\begin{figure}
    \centering
    \begin{tabular}{c}
    \includegraphics[width=0.5\linewidth]{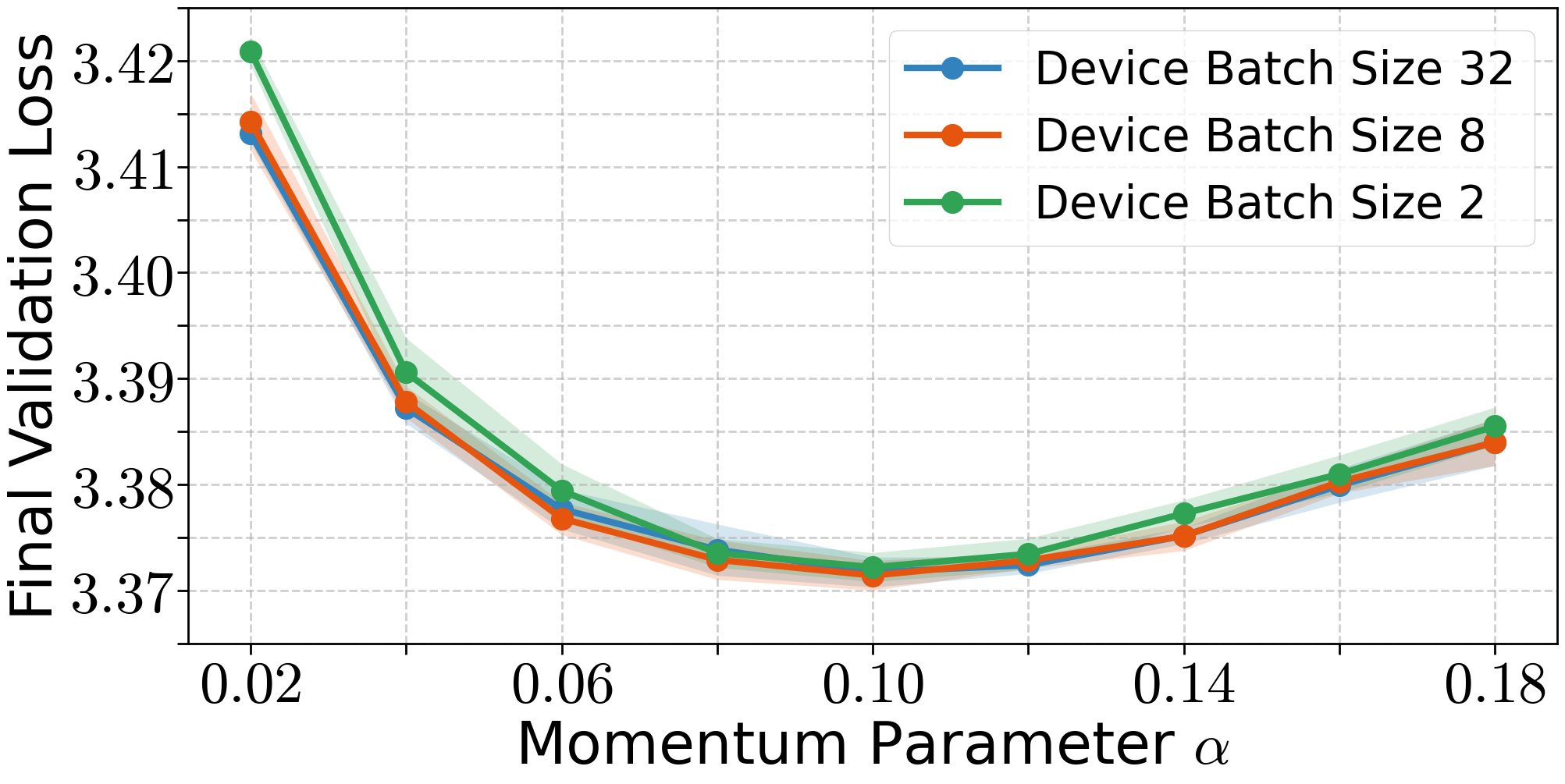} 
    
    \end{tabular}
    \caption{Final performance of 124M model with $B=256, S=1024, \beta=3.6\cdot10^{-4}$ and varying the momentum parameter $\alpha$ and device batch size.
    }
    \label{fig:quantization}
\end{figure}

\newpage
\section{In-Expectation Convergence Proofs for SCG}\label{apx:inexp_convergence_proofs_no_restarts}

The proof structure is inspired by the analysis of the first-order stochastic trust-region method with momentum under star-convexity by \citet{kovalev2025orthogonalization}.

\begin{lemma}\label{lem:bound_x_mu_kl_no_restarts} Let assumptions \eqref{eq:smoothness} and \eqref{eq:mu_kl} hold. Assume that $x_0$ and $\eta$ are chosen such that  
\begin{equation}\label{eq:lemma1_1_mu_kl_no_restarts}
    2\|x_0\| \le \eta, \quad \beta = \frac{c}{K}, \quad \text{and} \quad K\ge 2c.
\end{equation}
Let $\{x_k\}$ be the iterates of \Cref{alg:spectral_gd_decay_fw} . Then, the following inequalities hold for all $k\in\{0,1\ldots,K-1\}$
\begin{align}\label{eq:bound_x_mu_kl_no_restarts} 
    \eta - \|x_k\| \ge (1-\beta)^{k}\frac{\eta}{2}, \quad  \|x_{k+1}-x_k\|\le 2\beta\eta.
\end{align}
\end{lemma}
\begin{proof}
    We show by induction $k$ that 
    $$\|x_k\| \le (1-\beta)^{k}\frac{\eta}{2} + \eta(1-(1-\beta)^{k})$$
    holds. The base of induction is $k=0$. Note that $\|x_0\| \le \frac{\eta}{2}$ holds by the choice of $\eta$ and $x_0$ in \eqref{eq:lemma1_1_mu_kl_no_restarts}. Assume that inequalities hold for some $k \in\{0, 1,\ldots K - 2\}$. We show that they also hold at iteration $k+1$. Indeed, we have 
    \begin{align*}
      \|x_{k+1}\|
      &\aeq{uses the update step}
      \|(1-\beta)x_k + \beta\eta d_{k+1}\|
      \aleq{uses the triangle inequality}
      (1-\beta)\|x_k\| + \beta\eta\|d_{k+1}\|\\
     & \aleq{uses the restriction on $d_{k+1}$ in  and induction hypothesis}
      (1-\beta)\left((1-\beta)^{k}\frac{\eta}{2} + \eta(1-(1-\beta)^{k})\right)
      + \beta\eta\\
      &=(1-\beta)^{k+1}\frac{\eta}{2} + \eta((1-\beta) - (1-\beta)^{k+1} + \beta) = (1-\beta)^{k+1}\frac{\eta}{2} + \eta(1 - (1-\beta)^{k+1}),
    \end{align*}
    where \annotate. This concludes the induction step and proves the first inequality in \eqref{eq:bound_x_mu_kl_no_restarts} for all $k\in\{0,1,\ldots,K\}.$ We can lower bound $(1-\beta)^K$ using the inequality $\log(1-y)\ge -y-y^2$ for all $y\in[0,0.5]$ and $K\ge 2c$ as follows
    \begin{equation}
        \log((1-\beta)^K) = K\log(1-\beta) \ge K(-\beta-\beta^2)= -c-\frac{c^2}{K} \ge -\frac{3c}{2}.
    \end{equation}
    This implies that $(1-\beta)^K \ge e^{-\nicefrac{3c}{2}}$. Using the obtained bound, we derive for all $k\in\{0,1,\ldots,K\}$
    \begin{align}\label{eq:x_k_bound_by_eta}
        \|x_{k}\| &\le (1-\beta)^{K}\frac{\eta}{2} + \eta(1-(1-\beta)^{K})
        \le \eta - \frac{\eta}{2}e^{-\nicefrac{3c}{2}}.
    \end{align}
    Now we prove the last inequality in \eqref{eq:bound_x_mu_kl_no_restarts}. We have
    \begin{align*}
        \|x_{k+1} - x_k\| &\aeq{uses the update rule} \|-\beta x_k + \beta\eta d_{k+1}\| \\
        &\aleq{uses the triangle inequality} \beta\|x_k\| + \beta\eta\|d_{k+1}\| \\
        &\aleq{uses the previous inequality and the restriction on $d_{k+1}$} \beta((1-\beta)^{K}\frac{\eta}{2} + \eta(1-(1-\beta)^{K})) + \beta\eta\\
        &=\beta\eta((1-\beta)^{K}/2 + 1-(1-\beta)^{K} + 1) \\
        &= \beta\eta(2 - (1-\beta)^{K}/2) \le 2\beta\eta,
    \end{align*}
    where \annotate.
\end{proof}

\begin{lemma}\label{lem:momentum_decay_inexp_fw_no_restarts}
  Let Assumptions \eqref{eq:smoothness}, \eqref{eq:norm_equiv} and \eqref{eq:bounded_variance} hold. Let $m_{0}=g(x_0;\xi_0)$, then the iterates of \Cref{alg:spectral_gd_decay_fw}  satisfy the following inequality:
  \begin{equation}
    \EE[\norm{m_{k+1} - \nabla f(x_{k})}_*] \leq (1-\alpha)^{k}\rho\sigma + \frac{2L\beta\eta}{\alpha}
  + \rho \sigma\sqrt{\alpha}. \notag
  \end{equation}
\end{lemma}
\begin{proof}
    We can express $m_{k+1} - \nabla f(x_{k})$ as follows using the definition of the momentum buffer in \Cref{alg:spectral_gd_decay_fw}
\begin{align*}
  m_{k+1} - \nabla f(x_k)
  &=
  (1-\alpha)m_k + \alpha g(x_k;\xi_k) - \nabla f(x_k)
  \\&=
  (1-\alpha)(m_k - \nabla f(x_{k-1})) + \alpha(g(x_k;\xi_k) - \nabla f(x_k))
  \\&
  +(1-\alpha)(\nabla f(x_{k-1}) - \nabla f(x_k)).
\end{align*}
This implies the following for all $k \geq 0$:
\begin{align*}
  m_{k+1} - \nabla f(x_k)
  &=
  (1-\alpha)^{k}(m_{1} - \nabla f(x_0))
  + \sum_{i=0}^{k-1}(1-\alpha)^{k-i}(\nabla f(x_{i}) - \nabla f(x_{i+1}))
  \\&
  +\sum_{i=1}^{k}\alpha(1-\alpha)^{k-i}(g(x_{i},\xi_{i}) - \nabla f(x_{i})).
\end{align*}
Using this decomposition, we can upper-bound $\norm{m_{k+1} - \nabla f(x_k)}_*$  as follows
\begin{align*}
  \norm{m_{k+1} - \nabla f(x_k)}_*
  &\aleq{uses the triangle inequality}
  (1-\alpha)^{k}\norm{m_{1} - \nabla f(x_{0})}_* + \sum_{i=0}^{k-1}(1-\alpha)^{k-i}\norm{\nabla f(x_{i}) - \nabla f(x_{i+1})}_*
  \\&
  \quad+\norm{\sum_{i=1}^{k}\alpha(1-\alpha)^{k-i}(g(x_i,\xi_i) - \nabla f(x_i))}_*
  \\&\aleq{uses \eqref{eq:smoothness} and \Cref{lem:bound_x_mu_kl_no_restarts} with $L, \beta, \eta$}
  (1-\alpha)^{k}\norm{m_{1} - \nabla f(x_{0})}_*\\
  &\quad+ \sum_{i=0}^{k-1}(1-\alpha)^{k-i}2L\beta\eta+\norm{\sum_{i=1}^{k}\alpha(1-\alpha)^{k-i}(g(x_i,\xi_i) - \nabla f(x_i))}_*
  \\&\aleq{uses \eqref{eq:norm_equiv}}
  (1-\alpha)^{k}\rho\norm{m_{1} - \nabla f(x_{0})}_2
  + \sum_{i=0}^{k-1}(1-\alpha)^{k-i}2L\beta\eta
  \\&\quad 
  +\rho\norm{\sum_{i=1}^{k}\alpha(1-\alpha)^{k-i}(g(x_i,\xi_i) - \nabla f(x_i))}_2,
\end{align*}
where \annotate. Next, we take the full expectation and get
\begin{align*}
    \EE\left[\norm{m_{k+1} - \nabla f(x_k)}_*\right] &\leq (1-\alpha)^{k}\rho\EE[\norm{m_{1} - \nabla f(x_{0})}_2]
  + \sum_{i=0}^{k-1}(1-\alpha)^{k-i}2L\beta\eta
  \\&\quad 
  +\rho\EE\left[\norm{\sum_{i=1}^{k}\alpha(1-\alpha)^{k-i}(g(x_i,\xi_i) - \nabla f(x_i))}_2\right]\\
  &\aleq{uses Jensen's inequality} (1-\alpha)^{k}\rho\sqrt{\EE[\norm{m_{1} - \nabla f(x_{0})}_2^2]}
  + \sum_{i=0}^{k-1}(1-\alpha)^{k-i}2L\beta\eta
  \\&\quad 
  +\rho\sqrt{\EE\left[\norm{\sum_{i=1}^{k}\alpha(1-\alpha)^{k-i}(g(x_i,\xi_i) - \nabla f(x_i))}_2^2\right]}\\
  &\aleq{uses \eqref{eq:bounded_variance} and the fact that samples $\xi_i \sim \mathcal{D}$ are independent} (1-\alpha)^{k}\rho\sigma + \sum_{i=0}^{k-1}(1-\alpha)^{k-i}2L\beta\eta + \alpha \rho\sigma \sqrt{\sum\limits_{i=1}^k(1-\alpha)^{2(k-i)}}\\
  &\leq (1-\alpha)^{k}\rho\sigma + \frac{2L\beta\eta}{\alpha} + \sqrt{\alpha}\rho\sigma,
\end{align*}
where \annotate.
\end{proof}

\begin{theorem}[Full statement of \Cref{thm:str_decay_mu_kl_expectation_no_restarts}]\label{thm:str_decay_mu_kl_expectation_no_restarts_full_statement} Let Assumption \eqref{eq:smoothness}, \eqref{eq:norm_equiv}, \eqref{eq:mu_kl}, \eqref{eq:bounded_variance} hold. Let $m_{0} = g(x_{0};\xi_{0})$
and $c > 0$. Let the parameters of \Cref{alg:spectral_gd_decay_fw}  are chosen as follows
\begin{align}\label{eq:choice_parameters_1_no_restarts_inexp}
    \beta=\frac{c}{K}, \quad 
    \eta = \frac{2e^{3c/2}}{\mu c}\log\left(\frac{2(f(x_0)-f^{\star})}{\varepsilon}\right), 
    \quad 2\|x_0\| \le \eta,
\end{align}
and 
\begin{align}\label{eq:choice_parameters_2_no_restarts_inexp}
    \alpha &= \min\left\{1, \frac{(\varepsilon\mu)^2}{(32\rho\sigma)^2e^{3c}} \right\},\\
    K &= \max\left[
    2c,
     \max\left\{\frac{1}{2}, \frac{128Le^{3c}}{\varepsilon\mu^2},
    \frac{32\rho\sigma e^{\nicefrac{3c}{2}}}{\varepsilon\mu},
    \frac{128Le^{6c}(32\rho\sigma)^2}{\mu(\varepsilon\mu)^3},
    \frac{(32\rho\sigma e^{\nicefrac{3c}{2}})^3}{(\varepsilon\mu)^3}\right\}\log\left(\frac{2(f(x_0)-f^{\star})}{\varepsilon}\right) 
    \right].\notag
\end{align}
Then the output of \Cref{alg:spectral_gd_decay_fw} after $K$ iterations satisfies $\EE[f(x_K) - f^{\star}] \le \varepsilon$.
\end{theorem}
\begin{remark}
    The choice of $\eta \sim \log\left(\frac{2(f(x_0)-f^{\star})}{\varepsilon}\right)$ and $2\|x_0\| \leq \eta$ ensures a sufficient contraction factor in front of $f(x_k)-f^\star$ in the proof. We also note that all iterates produced by Algorithm~\ref{alg:spectral_gd_decay_fw} have a bounded norm by $\eta$. However, we do not make any explicit assumptions about $\arg\min_{x\in\cX} f(x)$, e.g., we do not assume its existence or boundedness of its norm by $\eta$. Therefore, for a fixed choice of $\varepsilon$ it is possible that an optimizer has norm larger than $\eta$ while $\EE[f(x_K) - f^{\star}] \le \varepsilon$. 
\end{remark}

\begin{proof}
    Let $u_k = {\rm arg}\min_{u\in\cX}\<\nabla f(x_k), u>$ s.t. $\|u\|\le 1$. Then we have
    \begin{align}
      f(x_{k+1}) 
      &\aleq{uses \eqref{eq:smoothness}}
      f(x_k) + \<\nabla f(x_k),x_{k+1} - x_k> + \frac{1}{2}L\sqn{x_{k+1} - x_k}
      \notag\\ &\aeq{uses the update step and \Cref{lem:bound_x_mu_kl_no_restarts}} f(x_k) + \<\nabla f(x_k), -\beta x_k + \beta\eta d_{k+1}> + 2L\beta^2\eta^2
      \notag \\&=
      f(x_k) -\beta\<\nabla f(x_k), x_k> + \beta\eta\<\nabla f(x_k) - m_{k+1},d_{k+1}> + \beta\eta\<m_{k+1},d_{k+1}> + 2L\beta^2\eta^2
      \notag\\ &\aleq{uses the optimality of $d_{k+1}$} 
      f(x_k) -\beta\<\nabla f(x_k), x_k> + \beta\eta\<\nabla f(x_k)-m_{k+1}, d_{k+1}> + \beta\eta\<m_{k+1},u_{k}> + 2L\beta^2\eta^2
      \notag\\ &\aeq{uses $\<\nabla f(x_k),u_k> = -\|\nabla f(x_k)\|_*$}
      f(x_k) -\beta\<\nabla f(x_k), x_k> + \beta\eta\<\nabla f(x_k)-m_{k+1}, d_{k+1} - u_{k}> - \beta\eta\|\nabla f(x_k)\|_* + 2L\beta^2\eta^2
      \notag\\&\aleq{uses Cauchy-Schwarz and $\|d_{k+1}\|,\|u_k\|\le 1$} f(x_k) + \beta\|\nabla f(x_k)\|_*\cdot\|x_k\| + 2\beta\eta\|\nabla f(x_k)-m_{k+1}\|_* - \beta\eta\|\nabla f(x_k)\|_* + 2L\beta^2\eta^2
      \notag\\ 
      &= f(x_k) -\beta\|\nabla f(x_k)\|_*(\eta-\|x_k\|)+ 2\beta\eta\|m_{k+1}-\nabla f(x_k)\|_*  + 2L\beta^2\eta^2
      \notag\\ 
      &\aleq{uses \Cref{lem:bound_x_mu_kl_no_restarts}, \eqref{eq:mu_kl}, and \eqref{eq:x_k_bound_by_eta}} 
      f(x_k) -\frac{\beta\eta\mu}{2}e^{-\nicefrac{3c}{2}}(f(x_k)-f^{\star}) + 2\beta\eta\|m_{k+1}-\nabla f(x_k)\|_* + 2L\beta^2\eta^2,\label{eq:intermediate_decrease_simplified_no_restarts_inexp}
    \end{align}
    where \annotate. With the assumption that $m_{0} = g(x_{0};\xi_{0})$, we have from \Cref{lem:momentum_decay_inexp_fw_no_restarts} that 
    \begin{equation*}
        \EE[\|m_{k+1} - \nabla f(x_k)\|_*] \le (1-\alpha)^{k}\rho\sigma + \frac{2L\beta\eta}{\alpha} + \rho\sigma\sqrt{\alpha}.
    \end{equation*}
Taking the expectation from \eqref{eq:intermediate_decrease_simplified_no_restarts_inexp} and using this bound and \Cref{lem:bound_x_mu_kl_no_restarts}, we derive
\begin{align}\label{eq:one_iteration_progress_no_restarts_inexp}
  \EE[f(x_{k+1}) - f^{\star}] &\le \left(1-\frac{\mu\beta\eta}{2e^{\nicefrac{3c}{2}}}\right)\EE[f(x_k)-f^{\star}] + (1-\alpha)^{k}2\beta\eta\rho\sigma + \frac{4L\beta^2\eta^2}{\alpha} + 2\beta\eta\rho\sigma\sqrt{\alpha} \notag\\
  &\hspace{10cm}+ 2L\beta^2\eta^2.
\end{align}
The contraction factor $1- \frac{\mu\beta\eta}{2e^{\nicefrac{3c}{2}}} \in (0,1)$ by the choice of $K \ge \frac{1}{2}\log\left(\frac{2(f(x_0)-f^*)}{\varepsilon}\right)$. Unrolling this recursion for all iterations $k\in\{0,1, \ldots, K-1\}$ and using the bound for the geometric series, we guarantee progress such that
\begin{align}
  \EE[f(x_{K}) - f(x^{\star})] &\leq \left(1-\frac{\mu\beta\eta}{2e^{\nicefrac{3c}{2}}}\right)^{K} (f(x_0) - f(x^{\star})) 
  + \frac{2\beta\eta\rho\sigma}{\alpha} 
  + \frac{4\rho\sigma\sqrt{\alpha}}{\mu}e^{\nicefrac{3c}{2}} \notag\\
  &\hspace{8cm}+ \frac{4L\beta\eta}{\mu}e^{\nicefrac{3c}{2}}
  + \frac{8L\beta\eta}{\alpha\mu}e^{\nicefrac{3c}{2}}.\label{eq:one_stage_progress_no_restarts_inexp}
\end{align}
Now we need to bound each of the terms proportionally to $\varepsilon$ using the choice of parameters $\eta,\alpha,\beta, K$ from \eqref{eq:choice_parameters_1_no_restarts_inexp} and \eqref{eq:choice_parameters_2_no_restarts_inexp}. First, we want
\begin{align*}
    &4\rho\sigma\frac{\sqrt{\alpha}}{\mu}e^{\nicefrac{3c}{2}} \le \frac{\varepsilon}{8} \Rightarrow \alpha \le \frac{(\varepsilon\mu)^2}{(32\rho\sigma)^2e^{3c}}.
\end{align*}
We can satisfy the above bound with the choice of $\alpha$ such that
\begin{align}
    \alpha = \min\left\{1, \frac{(\varepsilon\mu)^2}{(32\rho\sigma)^2e^{3c}} \right\},\label{eq:choice_alpha_high_prob_no_restarts_inexp}
\end{align}
which is exactly the choice of $\alpha$ in \eqref{eq:choice_parameters_2_no_restarts_inexp}. Next, we want 
\begin{align*}
    &\frac{8Le^{\nicefrac{3c}{2}}}{\mu}\frac{\beta\eta}{\alpha} 
    \le \frac{\varepsilon}{8}
    \Rightarrow \beta = \frac{c}{K} \le \frac{\varepsilon\mu\alpha}{64L\eta e^{\nicefrac{3c}{2}}} 
    \aleq{uses \eqref{eq:choice_alpha_high_prob_no_restarts_inexp}} \min\left\{\frac{\varepsilon\mu}{64L\eta e^{\nicefrac{3c}{2}}}, \frac{(\varepsilon\mu)^3}{64L\eta e^{\nicefrac{9c}{2}}(32\rho\sigma)^2}\right\},
\end{align*}
where \annotate. The above can be satisfied if we choose $K$ such that 
\begin{align}
    K &\ge \max\left\{\frac{64L\eta c e^{\nicefrac{3c}{2}}}{\varepsilon\mu}, \frac{64L\eta ce^{\nicefrac{9c}{2}}(32\rho\sigma)^2}{(\varepsilon\mu)^3}\right\},\notag\\
    &\aeq{uses the value of $\eta$} 
    \max\left\{\frac{128L e^{3c}}{\varepsilon\mu^2}, \frac{128Le^{6c}(32\rho\sigma)^2}{\mu(\varepsilon\mu)^3}\right\}\cdot \log\left(\frac{2(f(x_0)-f^{\star})}{\varepsilon}\right)
    \label{eq:choice_beta_1_inexp_no_restarts_inexp},
\end{align}
where \annotate. This choice of $K$ is satisfied by the choice in \eqref{eq:choice_parameters_2_no_restarts_inexp}. Moving on, we want 
\begin{equation}\label{eq:choice_beta_2_inexp_desire_no_restarts}
    \frac{4L\beta\eta e^{\nicefrac{3c}{2}}}{\mu} \le \frac{\varepsilon}{8} \Rightarrow \beta = \frac{c}{K} \le \frac{\varepsilon\mu}{32L\eta e^{\nicefrac{3c}{2}}}.
\end{equation}
We can satisfy the right inequality above if we choose $K$ such that
\begin{equation}\label{eq:choice_beta_2_inexp_no_restarts_inexp}
K \ge \frac{32L\eta ce^{\nicefrac{3c}{2}}}{\varepsilon\mu} 
\aeq{uses the value of $\eta$ in \eqref{eq:choice_parameters_1_no_restarts_inexp}} 
\frac{64L e^{3c}}{\mu^2\varepsilon}\log\left(\frac{2(f(x_0)-f^{\star})}{\varepsilon}\right),
\end{equation}
where \annotate. This choice of $K$ is satisfied by the choice in \eqref{eq:choice_parameters_2_no_restarts_inexp}. Finally, we want 
\begin{align*}
    2\rho\sigma\frac{\beta\eta}{\alpha} \le \frac{\varepsilon}{8} \Rightarrow \beta = \frac{c}{K} &\le 
    \frac{\varepsilon\alpha}{16\rho\sigma\eta} 
    \aleq{uses \eqref{eq:choice_alpha_high_prob_no_restarts_inexp}} 
    \min\left\{
    \frac{\varepsilon}{16\rho\sigma\eta}, 
    \frac{\varepsilon(\varepsilon\mu)^2}{16\rho\sigma\eta(32\rho\sigma)^2e^{3c}}
    \right\},
\end{align*}
where \annotate. The above inequality is satisfied with the choice of $K$ such that 
\begin{align}\label{eq:choice_K_3_inexp_no_restarts_inexp}
    K &\ge \max\left\{
    \frac{16\rho\sigma\eta c}{\varepsilon}, 
    \frac{16\rho\sigma\eta c(32\rho\sigma)^2e^{3c}}{\varepsilon(\varepsilon\mu)^2}
    \right\}\\
    &\aeq{uses the value of $\eta$ in \eqref{eq:choice_parameters_1_no_restarts_inexp}} 
    \max\left\{
    \frac{32\rho\sigma e^{\nicefrac{3c}{2}}}{\varepsilon\mu}, 
    \frac{(32\rho\sigma e^{\nicefrac{3c}{2}})^3}{(\varepsilon\mu)^3}
    \right\}\log\left(\frac{2(f(x_0)-f^{\star})}{\varepsilon}\right)
    ,\notag
\end{align}
where \annotate. This bound on $K$ is satisfied by the choice in \eqref{eq:choice_parameters_1_no_restarts_inexp}. A combination of \eqref{eq:choice_beta_1_inexp_no_restarts_inexp}, \eqref{eq:choice_beta_2_inexp_no_restarts_inexp}, \eqref{eq:choice_K_3_inexp_no_restarts_inexp} gives the choice of $K$ in \eqref{eq:choice_parameters_2_no_restarts_inexp}:
\begin{align}\label{eq:choice_beta_final_high_prob_no_restarts_inexp}
    K &= 
    \max\left\{\frac{128Le^{3c}}{\varepsilon\mu^2},
    \frac{32\rho\sigma e^{\nicefrac{3c}{2}}}{\varepsilon\mu},
    \frac{128Le^{6c}(32\rho\sigma)^2}{\mu(\varepsilon\mu)^3},
    \frac{(32\rho\sigma e^{\nicefrac{3c}{2}})^3}{(\varepsilon\mu)^3}\right\}\log\left(\frac{2(f(x_0)-f^{\star})}{\varepsilon}\right) 
    .
\end{align}
Now we show that the choice of $K$, $\beta$, and $\eta$ ensures that the first term in \eqref{eq:one_stage_progress_no_restarts_inexp} is smaller than $\varepsilon/2$. Let us show that  
\begin{equation}\label{eq:choice_K_desire_no_restarts_inexp}
    \left(1-\frac{\mu\beta\eta}{2e^{\nicefrac{3c}{2}}}\right)^{K}(f(x_0) - f^{\star}) \le e^{-\mu\beta\eta e^{-\nicefrac{3c}{2}}K/2}(f(x_0)-f^{\star}) \le \frac{\varepsilon}{2}.
\end{equation}
The last inequality is satisfied if the following condition holds: 
\begin{equation*}
    \frac{\mu\beta\eta}{2e^{\nicefrac{3c}{2}}}K \ge \log\left(\frac{2(f(x_0)-f^{\star})}{\varepsilon}\right).
\end{equation*}
Plugging in the choice of $\beta = \frac{c}{K}$ and $\eta = \frac{2e^{\nicefrac{3c}{2}}}{\mu c}\log\left(\frac{2(f(x_0)-f^{\star})}{\varepsilon}\right)$, we obtain
\begin{equation*}
    \frac{\mu\beta\eta}{2e^{\nicefrac{3c}{2}}}K = \frac{\mu}{2e^{\nicefrac{3c}{2}}}\cdot \frac{c}{K} \cdot \frac{2e^{\nicefrac{3c}{2}}}{\mu c}\log\left(\frac{2(f(x_0)-f^{\star})}{\varepsilon}\right) \cdot K = \log\left(\frac{2(f(x_0)-f^{\star})}{\varepsilon}\right).
\end{equation*}
Grouping the bounds together, we obtain that the choice of $K,\eta,\beta$ implies that 
$$
\EE[f(x_{K}) - f^{\star}] \le \frac{\varepsilon}{2} + 4\cdot\frac{\varepsilon}{8} = \varepsilon.
$$
\end{proof}

\begin{corollary}[Full statement of \Cref{cor:best_eps}]\label{cor:best_eps_appendix}
Under the setup of \Cref{thm:str_decay_mu_kl_expectation_no_restarts_full_statement}, let the token budget be large enough: $T \ge \max\left\{2cBS, \frac{BS}{2}\log\left(\frac{2(f(x_0)-f^*}{\varepsilon}\right)\right\}$. Then, running the algorithm with parameters from \Cref{thm:str_decay_mu_kl_expectation_no_restarts_full_statement} for $K=\nicefrac{T}{BS}$ iterations, we achieve the optimization error
\begin{equation}\label{eq:best_eps_full_statement}
    \varepsilon = \max\left\{
    \frac{128LBS e^{3c}}{\mu^2T},  
    \left(\frac{128Le^{6c}(32\rho\sigma_\star)^2}{\mu^4T}\right)^{ 1/3}, 
    \frac{32e^{\nicefrac{3c}{2}}\rho\sigma_\star}{\mu(T^2BS)^{1/6}}\right\}.
\end{equation}
\end{corollary}

\begin{proof}
    From \Cref{thm:str_decay_mu_kl_expectation_no_restarts_full_statement}, we have that to achieve the optimization error $\varepsilon$, we need to use $K$ iterations defined as 
    $$K =  
    \max\left[\frac{128Le^{3c}}{\varepsilon\mu^2}, \frac{32e^{\nicefrac{3c}{2}}\rho\sigma}{\varepsilon\mu}, 
    \frac{128Le^{6c}(32\rho\sigma)^2}{\mu(\varepsilon\mu)^3}, 
    \frac{(32\rho\sigma e^{\nicefrac{3c}{2}})^3}{(\varepsilon\mu)^3}\right]\log\left(\frac{2(f(x_0)-f^{\star})}{\varepsilon}\right),$$
    ignoring the requirements $K\ge 2c$, $K \ge \frac{1}{2}\log\left(\frac{2(f(x_0)-f^*)}{\varepsilon}\right)$, which holds in practice (and also follows from the assumption on $T$). Multiplying both sides of this expression by $BS$, using \Cref{asmp:bounded_variance} that says that $\sigma^2 = \frac{\sigma_\star^2}{BS}$, and using the relation $T=KBS$, we obtain
    $$T = \max\left\{ 
    \frac{128Le^{3c}LBS}{\varepsilon\mu^2}, \frac{32e^{\nicefrac{3c}{2}}\rho\sigma_\star\sqrt{BS}}{\varepsilon\mu}, \frac{128Le^{6c}(32\rho\sigma_\star)^2}{\mu(\varepsilon\mu)^3}, \frac{(32\rho\sigma_\star e^{\nicefrac{3c}{2}})^3}{(\varepsilon\mu)^3\sqrt{BS}}\right\}\log\left(\frac{2(f(x_0)-f^{\star})}{\varepsilon}\right).$$
    Since the token budget $T$ is fixed in the experiments, the expression above says that we cannot achieve an arbitrary optimization error $\varepsilon$:
    \begin{align*}\varepsilon &= \max\left\{ 
    \frac{128e^{3c}LBS}{T\mu^2}, 
    \frac{32e^{\nicefrac{3c}{2}}\rho\sigma_\star\sqrt{BS}}{T\mu}, \left(\frac{128Le^{6c}(32\rho\sigma_\star)^2}{\mu^4T}\right)^{1/3}, \right.\\
    &\hspace{9cm}\left.\left(\frac{(32\rho\sigma_\star e^{\nicefrac{3c}{2}})^3}{T\mu^3\sqrt{BS}}\right)^{1/3}\right\}\log\left(\frac{2(f(x_0)-f^{\star})}{\varepsilon}\right).
    \end{align*}
    Now we compare the second and fourth terms in the expression above. We note that the second term is larger \emph{iff}
    \begin{equation}
        \frac{32e^{\nicefrac{3c}{2}}\rho\sigma_\star\sqrt{BS}}{T\mu} \ge \frac{32\rho\sigma_\star e^{\nicefrac{3c}{2}}}{T^{1/3}\mu(BS)^{1/6}} \Longleftrightarrow (BS)^{2/3} \ge T^{2/3} \Longleftrightarrow BS \ge T.
    \end{equation}
    In other words, the second term is smaller than or equal to the fourth term. Therefore, it can be ignored in the maximum. This finalizes the proof.
\end{proof}

\begin{algorithm}[tb]
  \caption{Unconstrained Stochastic Conditional Gradient (uSCG)}
  \label{alg:spectral_gd_decay_fw_unconstrained}
  \begin{algorithmic}
    \STATE {\bfseries Input:} $x_0,m_0 \in \cX$, parameters $\alpha, \eta > 0$ 
    \FOR{$k=0, \ldots, K-1$}
    \STATE sample $\xi_k\sim\cD$
    \STATE compute $m_{k+1} = (1-\alpha)m_k + \alpha g(x_k;\xi_k)$
    \STATE compute $d_{k+1} = {\rm arg}\min_{d\in\cX}\<m_{k+1},d>$ s.t. $\|d\|\le 1$ 
    \STATE compute $x_{k+1} = x_k + \eta d_{k+1}$
    \ENDFOR
  \end{algorithmic}
\end{algorithm}

\begin{remark}
    Our work is based on the convergence guarantees under the $\mu$-KL condition following prior work \citep{schaipp2025surprising, islamov2024loss, tran2024reevaluating, guille2023no} that provides evidence that the loss landscape of neural networks exhibits a convex-like structure. However, it is possible to extend the results to a standard non-convex setting under the smoothness assumption only. In such a case, one can consider Unconstrained SCG (Algorithm~\ref{alg:spectral_gd_decay_fw_unconstrained}) and the convergence metric changes from the function sub-optimality to a dual gradient norm, i.e., $\min_{k=0,1,\ldots,K-1}\EE [\|\nabla f(x_k)\|_*]$ or $\EE[\|\nabla f(\overline{x}_k)\|_*]$ with $\overline{x}_k$ being selected uniformly at random from $\{x_0,x_1,\ldots,x_{K-1}\}$; see \citep[Theorem 5.5]{pethick2025scion} and \citep[Corollary 2]{kovalev2025orthogonalization}.
    
    Under the setup of \citep[Corollary 2]{kovalev2025orthogonalization} and a fixed token budget $T$, we achieve the optimization error
\begin{eqnarray}
    \varepsilon &=& \cO\left(\max\left\{
    \frac{\sqrt{L\Delta BS}}{\sqrt{T}},
    \frac{(L\Delta)^{1/4}\sqrt{\rho\sigma_{\star}}}{T^{1/4}}, 
    \frac{\rho\sigma_{\star}\sqrt{BS}}{T},
    \frac{\rho\sigma_{\star}}{T^{1/3}(BS)^{1/6}}
    \right\}\right)\notag\\
    &=& \cO\left(\max\left\{
    \frac{\sqrt{L\Delta BS}}{\sqrt{T}},
    \frac{(L\Delta)^{1/4}\sqrt{\rho\sigma_{\star}}}{T^{1/4}},
    \frac{\rho\sigma_{\star}}{T^{1/3}(BS)^{1/6}}
    \right\}\right), \label{eq:varepsilon_nonconvex}
\end{eqnarray}
where we used $\Delta = f(x_0) - f^\star$ and $T \ge BS$. 
We observe that the 
third term in \eqref{eq:varepsilon_nonconvex} is identical to the third term in \eqref{eq:best_eps} (up to constant $\mu$, which is expected due to the change of convergence metric), while the first two are different. Besides, the middle term is also batch size and sequence length independent, but has a power $T^{1/4}$ instead of $T^{1/3}$ as in \eqref{eq:best_eps}. Following the approach of \Cref{sec:strategies}, i.e., choosing $B$ and $S$ in the intersection of the first two terms in \eqref{eq:how_to_scale_batch_size_nonconvex}, we derive the scaling rules similar to \eqref{eq:how_to_scale_batch_size}
\begin{equation}\label{eq:how_to_scale_batch_size_nonconvex}
    B_1S_1 = B_0S_0 \sqrt{\frac{D_1}{D_0}\frac{\rho_1^2}{\rho_0^2} \frac{L_0}{L_1}},
\end{equation}
assuming that parameters $\Delta$ and $\sigma_\star$ are independent of the model size. This approach is similar to \eqref{eq:how_to_schedule_batch_size_LT} up to problem-dependent constants $\rho$ and $L$. 
\end{remark}

\end{document}